\documentclass[10pt]{article} 
\usepackage[preprint]{tmlr}

\usepackage{amsmath,amsfonts,bm}

\def\eqref#1{equation~\ref{#1}}

\def\1{\bm{1}}

\DeclareMathAlphabet{\mathsfit}{\encodingdefault}{\sfdefault}{m}{sl}
\SetMathAlphabet{\mathsfit}{bold}{\encodingdefault}{\sfdefault}{bx}{n}

\usepackage{hyperref}
\usepackage{url}
\usepackage{booktabs}
\usepackage{placeins}
\usepackage{graphicx}
\usepackage{subcaption}
\usepackage{algorithm}
\usepackage{algpseudocode}
\algnewcommand\Input{\item[\textbf{Input:}]}
\usepackage{amsthm}
\newtheorem{theorem}{Theorem}[section]
\newtheorem{proposition}[theorem]{Proposition}
\newtheorem{lemma}[theorem]{Lemma}
\newtheorem{corollary}[theorem]{Corollary}
\theoremstyle{definition}

\theoremstyle{remark}

\title{When Is a Draft Accepted?
\\ A Theory of Acceptance in Speculative Decoding}

\author{\name Aaryam Sharma\thanks{Website: \url{https://aaryam.info/}} \email a584shar@uwaterloo.ca \\
      \addr Independent Researcher\\
      Waterloo, Ontario, Canada}

\def\month{MM}  
\def\year{YYYY} 
\def\openreview{\url{https://openreview.net/forum?id=XXXX}} 

\begin{document}

\maketitle

\begin{abstract}
Speculative decoding accelerates language model inference by using a fast drafter to propose candidate tokens that are then verified by a larger target model. Existing theory largely studies the stochastic, distribution-preserving setting, where the goal is to exactly sample from the target distribution. In contrast, many practical systems use greedy decoding, relaxed acceptance rules, or tree-based candidate sets, where success is governed by local ranking and threshold events rather than exact distributional equality. We develop a theory for these regimes. We identify that many common acceptance criteria have rejection regions that can be characterized as lower level sets of the target distribution. For these, we characterize the exact KL divergence required for rejection yielding exact certificates and sharp margin-based bounds for strict greedy decoding, additive and multiplicative relaxed acceptance, top-(m) relaxed criteria, and entropy-thresholded acceptance. We then extend the framework to greedy tree decoding, deriving exact and margin-only certificates for when the target greedy token remains covered by the drafter's top-(m) candidates. Finally, we evaluate the resulting certificates on Qwen3 models, showing that relaxed and tree-based criteria substantially enlarge the region of certified acceptance, especially on decoding steps with low target model distribution margin. These results complement existing distribution-preserving analyses of speculative decoding by characterizing the deterministic local acceptance events common in practical inference systems.
\end{abstract}

\section{Introduction}

Most autoregressive LLMs generate only one token at a time. This means that generating $t$ tokens requires $t$ forward passes since each output token becomes part of the input for the next token. This is increasingly expensive as the size of deployed models grows. In a time when LLMs are being deployed in latency-sensitive applications, this bottleneck is a major concern.

Speculative decoding is a technique that addresses this limitation by introducing a smaller, faster draft model that proposes one or more future tokens that the larger target model verifies in parallel. The foundational speculative decoding and speculative sampling algorithms showed that this approach can significantly accelerate generation while maintaining the same output distribution as the target model \citep{leviathan2023fast, chen2023accelerating}.

These drafters range from lightweight n-gram proposals to trained neural models that share layers, hidden states, or embeddings with the target model, and drafts may be produced as a single sequence or as a tree of candidate continuations (Section~\ref{sec:related}). 

The performance of a speculative decoding technique is typically measured using the length of tokens accepted as well as latency and throughput gains. However, these metrics are not always measured in the same setting. While many speculative decoding methods aim to match the target model distribution exactly, a large focus of recent research has been on relaxing this requirement. These relaxed criteria furnish substantial practical performance benefits, yet have received comparatively little theoretical treatment.

Indeed, most theoretical analyses of speculative decoding today focus on the exact sampling setting, studying stochastic sampling, unbiasedness, or optimal transport \citep{yin2024a,sun2023spectr,ahn2023spectr}. This leaves open the deterministic, local questions that govern the methods used in practice: greedy decoding, relaxed greedy decoding, and tree-based greedy variants. 

This paper seeks to address this gap by providing a theoretical analysis of the local properties of speculative decoding. We begin by analyzing a family of acceptance methods under the bounds of KL divergence between target and draft output distributions. Our primary observation is that many practical acceptance rules, including greedy acceptance, relaxed acceptance, and even Medusa-style relaxed acceptance, can be characterized as lower level sets of the target distribution. We hence convert the problem of acceptance into a question about how much KL divergence is required before a draft can be rejected. We provide exact certificates such that if the target model and draft model KL divergences are smaller than the certificate we guarantee acceptance. We provide tight lower and upper bounds to these thresholds. We also similarly analyze a simple tree-based greedy decoding method relating the branching amount to local acceptance.

Finally, we compute statistics on model output distributions of Qwen3 models \citep{qwen3} and empirically analyze the bounds derived in our theoretical analysis.

Our contributions are as follows:
\begin{itemize}
    \item We develop a unified certificate for single draft token acceptance criteria whose rejection regions are lower level sets of the target distribution, obtaining bounds on strict greedy, additive relaxation, multiplicative relaxation, top-$m$ relaxed acceptance, and entropy-based acceptance as corollaries.
    \item We derive tight lower and upper bounds on the exact minimum KL required for rejection, and construct examples showing that these bounds are tight for various acceptance criteria.
    \item We extend the single draft token framework to greedy tree decoding when branching with factor $m$, deriving exact conditions for acceptance and general bounds. We also derive a universal $\log(m+1)$ upper bound on the tree certificate, extending the $\log(2)$ ceiling of the single-token case.
    \item We empirically evaluate the resulting certificates on Qwen3 models \citep{qwen3} and show how relaxed and tree-based criteria enlarge the region of certified acceptance, especially on low confidence steps.

\end{itemize}

\section{Related Works}
\label{sec:related}

Speculative execution predates modern language models, but its use for neural autoregressive decoding was developed in the context of blockwise parallel decoding by \citet{stern2018blockwise} who proposed predicting several future positions in parallel and accepting the longest prefix verified by the target model. The modern LLM formulation based on speculative sampling was introduced by \citet{leviathan2023fast} and \citet{chen2023accelerating} who showed that this method can exactly sample from the target model distribution while achieving significant speedups.

\paragraph{Architectural variants.}
Much subsequent research improves the algorithm through architectural changes to the drafter. Kangaroo and Draft \& Verify reuse parts of the target model itself within the draft model \citep{zhang2024draftverify,liu2024kangaroo}. Multi-head approaches such as Medusa and Hydra attach additional heads that predict future tokens based on the last target hidden state \citep{cai2024medusa,ankner2024hydra}. Models such as EAGLE, HASS, and Combined Token/Embedding Speculators use drafters that share hidden states or embeddings with the target model \citep{li2024eagle,li_eagle-2_2024,li2026eagle,wertheimer2024combined,zhang2025learning}, and parallel drafters such as DFlash use target hidden states to generate drafts in a single forward pass, while PARD-2 can run on both target model hidden states and standalone \citep{chen_dflash_2026,an2026pard,an_pard-2_2026}.

\paragraph{Acceptance criteria.}

Another popular research direction explores alternatives to exact-match acceptance. Medusa and Hydra relax exact matching with an entropy-based criterion that admits high-probability non-greedy tokens \citep{cai2024medusa,ankner2024hydra}. Popular inference libraries provide similar support: TensorRT-LLM uses a multiplicative margin for greedy decoding, while SGLang offers a multiplicative, sampling-based relaxation \citep{nvidia_tensorrt_llm_2026,zheng2024sglang,sglang_2026}. Fuzzy Speculative Decoding determines acceptance using divergences between the target and draft distributions \citep{holsman2025fuzzy}, and SPRINTER introduces a verifier that calls the target model only when the verifier rejects a draft \citep{zhong2025approxverify}. MARS additionally accounts for the margin between the top two target tokens, accepting the second-most-likely token when it is close enough to the most likely one \citep{song2026mars}. In another line of work, tree-based methods such as SpecInfer, EAGLE-2, and SEQUOIA explore multiple draft sequences in parallel where draft tokens are proposed and verified as a tree \citep{miao_specinfer_2024,li_eagle-2_2024,chen2024sequoia}.

\paragraph{Theoretical analysis.}
Most theoretical work targets the distribution-preserving setting. SpecTr formulates speculative decoding through optimal transport and studies transport-based acceptance rules \citep{sun2023spectr}, and SpecTr++ improves the resulting transport plan \citep{ahn2023spectr}. \citet{yin2024a} model speculative decoding as a Markov chain, derive expected-rejection formulas, prove the optimality of standard rejection-based speculative decoding within an unbiased algorithm class, and study batch decoding and the trade-off between inference cost and output quality. These results, close to ours in spirit, focus on stochastic output preservation and rejection-based efficiency. In contrast, we study deterministic local acceptance in strict greedy, relaxed acceptance, entropy-based, and tree-based criteria. Our analysis asks when local acceptance is stable when draft and target model distributions diverge.

\section{Preliminaries}

\subsection{Speculative Decoding}
\label{sec:prelim-spec-dec}

In this section, we provide a mathematical formulation for speculative decoding, and introduce the necessary notation and definitions to facilitate our theoretical analysis.

Let us define the current sequence of tokens as $x_0,...,x_n$. In a standard autoregressive decoding process, the model generates the next token $x_{n+1}$ based on the previous tokens $x_0,...,x_n$. This is highlighted in Algorithm 1. However, this technique is limited in that it can only generate one token at a time. On the other hand, a transformer forward pass can score multiple positions in parallel once a sequence is given. Speculative decoding exploits this by drafting several tokens cheaply and verifying them with one target-model call.

In speculative decoding, a smaller draft model is introduced. This model is typically significantly faster than the target model. Given the same input sequence, this model is used to generate a sequence of $k$ draft tokens $y_1,...,y_k$. The larger target model then verifies the correctness of these $k$ tokens in parallel, and accepts a prefix of these tokens based on a certain acceptance criterion. Often, the first rejected token is resampled from the target model, and if no tokens are rejected the target model generates an extra token.

Let us define $p(  \cdot | x_0,...,x_n)$ as the probability distribution of the next token $x_{n+1}$ for the target model, and $q( \cdot | x_0,...,x_n)$ as the analogous probability distribution for the draft model.

\begin{figure}[h]
\centering
\begin{minipage}[t]{0.47\textwidth}
\begin{algorithm}[H]
\caption{Autoregressive Decoding}
\label{alg:autoregressive}
\begin{algorithmic}[1]
\Input Prefix $x_{1:n}$, target model $p$
\For{$t = n+1, n+2, \ldots$}
    \State Sample $x_t \sim p(\cdot \mid x_{1:t-1})$
\EndFor
\end{algorithmic}
\end{algorithm}
\end{minipage}
\hfill
\begin{minipage}[t]{0.50\textwidth}
\begin{algorithm}[H]
\caption{Speculative Decoding (General)}
\label{alg:spec-dec-general}
\begin{algorithmic}[1]
\Input Prefix $x_{1:n}$, draft model $q$, target model $p$, draft length $k$
\State \textbf{Draft:} Sample $y_1, \ldots, y_k$ from $q$
\State \textbf{Verify:} Compute $p(\cdot \mid x_{1:n}, y_{1:i-1})$ for $i = 1, \ldots, k$
\State \textbf{Accept:} Accept a prefix $y_{1:\tau}$ for some $\tau \leq k$
\end{algorithmic}
\end{algorithm}
\end{minipage}
\end{figure}

While several algorithms exist for speculative decoding, among the most commonly used are \textit{speculative sampling} and \textit{greedy decoding}. Speculative sampling is a sampling approach that exactly simulates the target model's output distribution. Here, the draft tokens are sampled according to the draft model's probability distribution, and accepted based on a probabilistic acceptance criterion. Greedy decoding, on the other hand, is a deterministic approach wherein the draft tokens proposed are those with the highest probability, and a token is accepted iff it is the token with the highest probability under the target model distribution and its prefix has been accepted. For speculative sampling, we define for any two probability distributions $p$ and $q$ the distribution $[p - q]_+$ as the normalized probability distribution of $\max(0, p - q)$. A standard result establishes that under speculative sampling, the per-token acceptance probability equals $1 - \mathrm{TV}(p, q)$ \citep{leviathan2023fast,yin2024a}.

\begin{figure}[h]
\centering
\begin{minipage}[t]{0.47\textwidth}
\begin{algorithm}[H]
\caption{Speculative Sampling \citep{chen2023accelerating, leviathan2023fast}}
\label{alg:spec-sampling}
\begin{algorithmic}[1]
\Input Prefix $x_{1:n}$, draft model $q$, target model $p$, draft length $k$
\For{$i = 1, \ldots, k$}
    \State Sample $y_i \sim q(\cdot \mid x_{1:n}, y_{1:i-1})$
\EndFor
\For{$i = 1, \ldots, k$}
    \State Let $p_i(\cdot) := p(\cdot \mid x_{1:n}, y_{1:i-1})$, $q_i(\cdot) := q(\cdot \mid x_{1:n}, y_{1:i-1})$
    \State Sample $r \sim \mathrm{Uniform}[0,1]$
    \If{$r \leq \min\!\left(1,\, \dfrac{p_i(y_i)}{q_i(y_i)}\right)$}
        \State Accept $y_i$
    \Else
        \State Resample $y_i \sim [p_i - q_i]_+(\cdot)$. \textbf{Break.}
    \EndIf
\EndFor
\end{algorithmic}
\end{algorithm}
\end{minipage}
\hfill
\begin{minipage}[t]{0.50\textwidth}
\begin{algorithm}[H]
\caption{Greedy Speculative Decoding}
\label{alg:greedy-spec}
\begin{algorithmic}[1]
\Input Prefix $x_{1:n}$, draft model $q$, target model $p$, draft length $k$
\For{$i = 1, \ldots, k$}
    \State $y_i \leftarrow \arg\max_v\, q(\cdot \mid x_{1:n}, y_{1:i-1})$
\EndFor
\For{$i = 1, \ldots, k$}
    \State Let $p_i(\cdot) := p(\cdot \mid x_{1:n}, y_{1:i-1})$
    \If{$y_i = \arg\max_v\, p_i(v)$}
        \State Accept $y_i$
    \Else
        \State $y_i \leftarrow \arg\max_v\, p_i(v)$. \textbf{Break.}
    \EndIf
\EndFor
\end{algorithmic}
\end{algorithm}
\end{minipage}
\end{figure}

\subsection{Relaxed Greedy Acceptance}
\label{sec:prelim-relaxed}

Strict greedy acceptance requires the draft token to coincide with the target's argmax. Several practical systems relax this requirement, accepting draft tokens whose target probability is close to that of the argmax. We analyze two such criteria.

Let $p_i(\cdot) := p(\cdot \mid x_{1:n}, y_{1:i-1})$ denote the target distribution at step $i$, and let $x_i^\ast := \arg\max_v p_i(v)$ be the target's argmax. We consider:
\begin{itemize}
    \item \textbf{Additive relaxation} with margin $t \in [0,1]$: accept $y_i$ if $y_i = x_i^\ast$ or $p_i(y_i) > p_i(x_i^\ast) - t$.
    \item \textbf{Multiplicative relaxation} with factor $\alpha \in (0,1]$: accept $y_i$ if $y_i = x_i^\ast$ or $p_i(y_i) > \alpha \cdot p_i(x_i^\ast)$.
\end{itemize}
Strict greedy decoding is recovered by setting $t = 0$ or $\alpha = 1$. The multiplicative form is used in TensorRT-LLM relaxed acceptance \citep{nvidia_tensorrt_llm_2026}, while we will demonstrate that additive margins provide useful theoretical guarantees. Both criteria modify only the acceptance test in Algorithm~\ref{alg:greedy-spec}; the drafting step is unchanged.

\subsection{Tree-Based Greedy Decoding}
\label{sec:prelim-tree}

Tree-based speculative decoding generalizes the linear draft sequence to a tree of candidate continuations, allowing multiple drafts to be verified in parallel \citep{miao_specinfer_2024,li_eagle-2_2024,chen2024sequoia}. Our analysis focuses on a simple symmetric tree where at each level $i = 1, \ldots, k$, the drafter proposes the top $m$ tokens under the draft distribution producing $m^k$ candidate sequences.

Formally, let $\mathrm{top}_m(q_i)$ denote the set of $m$ tokens with highest probability under $q_i$. A path $(y_1, \ldots, y_k)$ in the tree is \emph{accepted} up to depth $\tau$ if, for every $i \leq \tau$, the target argmax $x_i^\ast$ lies in $\mathrm{top}_m(q_i)$ conditioned on $y_{1:i-1}$. The accepted token at each level is $x_i^\ast$ itself.

Setting $m = 1$ recovers strict greedy decoding. Larger $m$ trades verification cost (which scales with the number of tree nodes) for a higher chance that the target argmax is covered at each level.

\section{Single Draft Token Acceptance Criteria}
\label{sec:single-token-acceptance}

A draft token is accepted or rejected according to a criterion comparing it against the target distribution $p$. The acceptance criteria of several practical speculative decoders such as strict greedy decoding, additive and multiplicative relaxation, and entropy-thresholded acceptance, all reject a draft token exactly when it lies in a lower level set of $p$. In this section, we prove that a single exact KL certificate governs the acceptance of every such criterion. We apply this result to the individual acceptance criteria as corollaries from which we derive interpretable lower and upper bounds and show that they are tight by constructing examples.

\subsection{Setup}
\label{sec:single-setup}

Throughout this section, we fix a decoding step, thereby fixing our conditioning context. We thus suppress the dependence on the prefix $x_{1:n}$ since all the probabilities and results are conditioned on this context prefix. We write $p, q \in \Delta^{|V|-1}$ for the target and draft distributions over the vocabulary $V$, and
\[
x_0 \;:=\; \arg\max_{v \in V} p(v)
\]
for the target's argmax. We assume throughout that $x_0$ is unique. 

When $q$ has multiple modes, we assume the worst case for acceptance, that if any mode of $q$ lies in the rejection region, the drafter selects it. This makes our results valid under any deterministic tie-breaking rule even in the worst case.

A \emph{rejection region} is a set $\mathcal{R} \subseteq V \setminus \{x_0\}$ such that the draft is rejected precisely when $\arg\max_v q(v) \cap \mathcal{R} \ne \emptyset$. As KL divergence is a common training objective for speculative drafters, the quantity of interest is the minimum KL divergence that permits rejection,
\[
R_p(\mathcal{R}) \;:=\; \inf\bigl\{ \mathrm{KL}(p \,\Vert\, q) : \arg\max_v q(v) \cap \mathcal{R} \neq \emptyset \bigr\},
\]
This is the exact certificate for the acceptance criterion, since $\mathrm{KL}(p \,\Vert\, q) < R_p(\mathcal{R})$ guarantees acceptance for any $q$. In this section, our acceptance criteria are of the form $p(v) > c$ for some constant $c$ that may depend on $p$. Hence, $\mathcal{R} = \{v \in V : p(v) \leq c\}$ is a lower level set of $p$. The discussions below aim to explore the setting where $\mathrm{KL}(p \,\Vert\, q) \leq \epsilon$ holds.

Since target language models use a softmax output layer, we assume throughout this section that $p(v)>0$ for every $v\in V$. The arguments can be extended to distributions with zero coordinates by restricting to the support of $p$, but the full-support case is the relevant one for our setting.

\subsection{Exact Certificate for Single-Token Acceptance}

\begin{lemma}[Reduction to single-token problems]
\label{lem:single-reduction}
For any token $x_1 \in V \setminus \{x_0\}$, define
\[
M(x_1) \;:=\; \inf_{q \in \Delta^{|V|-1}}\; \mathrm{KL}(p \,\Vert\, q) \quad \text{subject to} \quad q(v) \leq q(x_1) \;\;\text{for all } v \in V.
\]
Then $R_p(\mathcal{R}) = \min_{x_1 \in \mathcal{R}} M(x_1)$.
\end{lemma}

\begin{proof}
The constraint in $M(x_1)$ forces $x_1$ to be a mode of $q$. If $\arg\max_v q(v) \cap \mathcal{R} \ne \emptyset$, then choosing $x_1$ to be that argmax shows $q$ is feasible for $M(x_1)$, so $\mathrm{KL}(p \,\Vert\, q) \geq M(x_1) \geq \min_{x_1 \in \mathcal{R}} M(x_1)$. 

Conversely, the minimizer of any $M(x_1)$ with $x_1 \in \mathcal{R}$ has $x_1$ as a mode, hence its argmax lies in $\mathcal{R}$, so it is feasible for $R_p(\mathcal{R})$. The two bounds coincide.
\end{proof}

\begin{lemma}[Single Token Minimizer]
\label{lem:single-kkt}
Fix $x_1 \in V \setminus \{x_0\}$. Let $c(x_1)$ be the unique level guaranteed by Lemma~\ref{lem:existence-of-c}, and set
\[
A(x_1) \;:=\; \{x_1\} \cup \{v \neq x_1 : p(v) > c(x_1)\}, \qquad s(x_1) \;:=\; \sum_{v \in A(x_1)} p(v).
\]
Then the minimizer of $M(x_1)$ is
\[
q^\ast(v) \;=\; \begin{cases} c(x_1) & v \in A(x_1), \\ p(v) & v \notin A(x_1), \end{cases}
\qquad\text{with}\qquad
c(x_1) \;=\; \frac{s(x_1)}{|A(x_1)|},
\]
and
\[
M(x_1) \;=\; s(x_1) \cdot \mathrm{KL}\!\left( \frac{p|_{A(x_1)}}{s(x_1)} \,\Bigm\Vert\, \mathrm{Unif}(A(x_1)) \right).
\]
\end{lemma}
\begin{proof}
First, we observe that the problem is convex, since KL is convex in its second argument and the constraints are linear. The interior of the simplex is strictly feasible, so KKT conditions are necessary and sufficient. 

With simplex multiplier $\mu$ and inequality multipliers $\lambda_v \geq 0$ for $v \neq x_1$, we obtain the Lagrangian as:
\[
\mathcal{L}(q, \lambda, \mu) \;=\; \mathrm{KL}(p \,\Vert\, q) + \sum_{v \neq x_1} \lambda_v\bigl(q(v) - q(x_1)\bigr) + \mu\Bigl(\sum_v q(v) - 1\Bigr)
\]
Stationarity gives the conditions
\[
-\frac{p(v)}{q(v)} + \lambda_v + \mu \;=\; 0 \quad \text{for } v \neq x_1, \qquad -\frac{p(x_1)}{q(x_1)} - \sum_{v \neq x_1} \lambda_v + \mu \;=\; 0.
\]
Let $c := c(x_1)$ and $A := A(x_1)$. Let us verify each KKT condition at $q^\ast$. For $v \notin A$, we have $q^\ast(v) = p(v)$, so stationarity forces $\lambda_v + \mu = 1$. The constraint $q^\ast(v) \leq q^\ast(x_1) = c$ holds since $p(v) \leq c$ for $v \notin A$. We set $\lambda_v = 0$, giving $\mu = 1$, and also satisfying complementary slackness for these inequality conditions.

For $v \in A \setminus \{x_1\}$, we have $q^\ast(v) = c$, and stationarity gives $\lambda_v = p(v)/c - \mu = p(v)/c - 1$. Since $v \in A \setminus \{x_1\}$ requires $p(v) > c$, we get $\lambda_v > 0$. The constraint $q^\ast(v) = q^\ast(x_1)$ is active, so complementary slackness holds.

The remaining stationarity condition at $x_1$ is
\[
\frac{p(x_1)}{c} \;=\; \mu - \sum_{v \neq x_1} \lambda_v \;=\; 1 - \sum_{v \in A \setminus \{x_1\}} \!\!\bigl(\frac{p(v)}{c} - 1\bigr).
\]
Multiplying through by $c$ and rearranging, this is equivalent to $|A| \cdot c = p(x_1) + \sum_{v \in A \setminus \{x_1\}} p(v) = \sum_{v \in A} p(v)$, which is the defining equation of $c$. Finally, $q^\ast$ satisfies the simplex constraint since by the same identity, $\sum_v q^\ast(v) = |A| \cdot c + \sum_{v \notin A} p(v) = \sum_{v \in A} p(v) + \sum_{v \notin A} p(v) = 1$.

At $q^\ast$, the terms for $v \notin A$ contribute zero to the KL since $q^\ast(v) = p(v)$. Writing $s := s(x_1)$ and $k := k(x_1) := |A|$, and using $c = s/k$,
\[
M(x_1) \;=\; \sum_{v \in A} p(v) \log\frac{p(v)}{c} \;=\; \sum_{v \in A} p(v) \log\frac{k \cdot p(v)}{s} \;=\; s \sum_{v \in A} \frac{p(v)}{s} \log\frac{p(v)/s}{1/k}.
\]
The final expression is exactly $s \cdot \mathrm{KL}\bigl(p|_A / s \,\Vert\, \mathrm{Unif}(A)\bigr)$.
\end{proof}

\begin{lemma}[Minimizing token]
\label{lem:single-swap}
$M(x_1)$ is non-increasing in $p(x_1)$. Consequently, if $\mathcal{R}$ has a most-probable element $x^\ast := \arg\max_{v \in \mathcal{R}} p(v)$, then $\min_{x_1 \in \mathcal{R}} M(x_1) = M(x^\ast)$.
\end{lemma}
\begin{proof}
Let $x_i, x_j$ with $p(x_i) \leq p(x_j)$, and let $q$ be the minimizer of $M(x_i)$ from Lemma~\ref{lem:single-kkt}. Since $x_i$ is a mode of $q$, $q(x_i) \geq q(x_j)$. Define $q'$ by swapping the values $q(x_i)$ and $q(x_j)$. Then $q'(x_j) = q(x_i) \geq q(v)$ for all $v$, so $x_j$ is a mode of $q'$ and $q'$ is feasible for $M(x_j)$. A direct computation gives
\[
\mathrm{KL}(p \,\Vert\, q) - \mathrm{KL}(p \,\Vert\, q') \;=\; \bigl(p(x_j) - p(x_i)\bigr) \log\frac{q(x_i)}{q(x_j)} \;\geq\; 0,
\]
as both factors are non-negative. Hence $M(x_j) \leq M(x_i)$, establishing monotonicity. The minimum over $\mathcal{R}$ is therefore attained at its most probable element $x^\ast$.
\end{proof}

\begin{theorem}[Exact certificate]
\label{thm:single}
Let $\mathcal{R} \subseteq V \setminus \{x_0\}$ be a rejection region with a most-probable element $x^\ast := \arg\max_{v \in \mathcal{R}} p(v)$, and let $c, A := A(x^\ast), s := s(x^\ast)$ be as in Lemma~\ref{lem:single-kkt}. Then
\[
R_p(\mathcal{R}) \;=\; s \cdot \mathrm{KL}\!\left( \frac{p|_A}{s} \,\Bigm\Vert\, \mathrm{Unif}(A) \right).
\]
In particular, $\mathrm{KL}(p \,\Vert\, q) < R_p(\mathcal{R})$ guarantees acceptance under the criterion whose rejection region is $\mathcal{R}$.
\end{theorem}

\begin{proof}
Combine Lemmas~\ref{lem:single-reduction}, \ref{lem:single-kkt}, and~\ref{lem:single-swap}: the reduction gives $R_p(\mathcal{R}) = \min_{x_1 \in \mathcal{R}} M(x_1)$, the swap lemma evaluates this at $x^\ast$, and the single-token minimizer lemma gives the closed form of $M(x^\ast)$.
\end{proof}

\subsection{General Lower Bound}

To obtain more interpretable bounds we show that the problem can be relaxed to a two token problem. For all our acceptance criteria below, this two-token relaxation gives tight worst-case bounds based on relevant margins and thresholds. While the exact certificate depends on $p$, the asymptotic behavior is captured by the two-token construction. 

Now, for $a, b > 0$ and $d \in [0,1)$, define
\[
G(a, b) \;:=\; a \log\frac{2a}{a+b} + b \log\frac{2b}{a+b}, \qquad g(d) \;:=\; \frac{1+d}{2}\log(1+d) + \frac{1-d}{2}\log(1-d),
\]
We can see that these are related by 
\[
G(a, b) = (a + b)\, g\left(\frac{a - b}{a + b}\right).
\]

\begin{lemma}[Two-token lower bound]
\label{lem:two-token-bound}
For any rejection region $\mathcal{R}$ with most-probable element $x^\ast$ under $p$,
\[
R_p(\mathcal{R}) \;\geq\; G\bigl(p(x_0),\, p(x^\ast)\bigr) \;=\; \bigl(p(x_0) + p(x^\ast)\bigr)\, g\!\left( \frac{p(x_0) - p(x^\ast)}{p(x_0) + p(x^\ast)} \right).
\]
Moreover $G$ is decreasing in its second argument: $\frac{\partial G}{\partial b} = \log\bigl(\frac{2b}{a+b}\bigr) < 0$ for $b < a$. Hence if $\mathcal{R} \subseteq \{v : p(v) \leq c\}$ for some threshold $c < p(x_0)$,
\[
R_p(\mathcal{R}) \;\geq\; G\bigl(p(x_0), c\bigr) \;=\; \bigl(p(x_0) + c\bigr)\, g\!\left( \frac{p(x_0) - c}{p(x_0) + c} \right).
\]
\end{lemma}

\begin{proof}
By Lemma~\ref{lem:single-reduction} and~\ref{lem:single-swap}, $R_p(\mathcal{R}) = M(x^\ast)$. Relax the defining problem of $M(x^\ast)$ by keeping only the single constraint $q(x_0) \leq q(x^\ast)$ and dropping all the other conditions of the form $q(v) \leq q(x^\ast)$. Fewer constraints can only lower the minimum. In this relaxed problem, the optimizer equalizes $q(x_0)$ and $q(x^\ast)$ and leaves all other coordinates unchanged, so the resulting value is exactly $G(p(x_0), p(x^\ast))$. The monotonicity of $G$ in its second argument follows by direct differentiation, and applying it with $p(x^\ast) \leq c$ gives the threshold form.
\end{proof}

\subsection{General Upper Bound}

In this section, for any single draft token acceptance criterion, we show a general upper bound on the certificates. We find that this is $\leq 2\log(2)/|A|$ where $A$ is the active set of tokens from Lemma~\ref{lem:single-kkt}. Since $|A| \geq 2$ for any rejection region, this implies a universal $\log(2)$ ceiling on the KL thresholds. Since every single token acceptance threshold is at most $\log(2)$, no criterion of this form can guarantee acceptance for all $q$ with $\mathrm{KL}(p \,\Vert\, q) > \log(2)$.

\begin{theorem}[Upper bound for single-token certificates]
\label{thm:log2-ceiling}
Let $\mathcal{R} \subseteq V \setminus \{x_0\}$ be any rejection region with a most-probable element, and let $R_p(\mathcal{R})$ be the exact certificate of Theorem~\ref{thm:single}. Let $k := |A|$ denote the size of the active set at the optimum. Then $k \geq 2$ and
\[
R_p(\mathcal{R}) \;\leq\; \frac{2}{k}\, \log 2 \;\leq\; \log 2,
\]
with the second inequality strict for $k \geq 3$. The bound is sharp in the limit for $k = 2$, $p(x_0) \to 1$, $p(x^\ast) \to 0$, both inequalities are attained.
\end{theorem}

\begin{proof}
By Theorem~\ref{thm:single} and Lemma~\ref{lem:single-kkt},
\[
R_p(\mathcal{R}) \;=\; s \cdot \mathrm{KL}\!\left( \frac{p|_A}{s} \,\Bigm\Vert\, \mathrm{Unif}(A) \right),
\]
where $A = \{x^\ast\} \cup \{v \neq x^\ast : p(v) > c\}$ is the active set used in computing $c$, $s := \sum_{v \in A} p(v)$. The active set always contains $x^\ast$ and $x_0$ so $k \geq 2$. 

Write $\nu \in \Delta^{k-1}$ for the distribution $\nu_v := p(v)/s$ on $A$. By construction of $A$, the $k - 1$ tokens $v \in A \setminus \{x^\ast\}$ have $p(v) > c = s/k$, hence $\nu_v > 1/k$, and for $x^\ast$ we have $\nu_{x^\ast} < 1/k$.
Thus $\nu$ lies in the constraint set
\[
\Omega_k \;:=\; \Bigl\{ \nu \in \Delta^{k-1} : \nu_v > \tfrac{1}{k} \text{ for } k - 1 \text{ indices},\; \nu_{x^\ast} < \tfrac{1}{k} \text{ for the remaining index} \Bigr\}.
\]
In terms of $\nu$ and $s$,
\[
R_p(\mathcal{R}) \;=\; s\bigl(\log k - H(\nu)\bigr).
\]
Since $s \leq 1$, it suffices to show $\log k - H(\nu) \leq (2/k) \log 2$ for all $\nu \in \overline{\Omega}_k$, where $\overline{\Omega}_k$ is the closure of $\Omega_k$ obtained by replacing strict inequalities with non-strict ones.

Now, the function $\log k - H(\nu)$ is convex in $\nu$ (since $H$ is concave), and $\overline{\Omega}_k$ is a polytope. Hence the maximum of $\log k - H(\nu)$ over $\overline{\Omega}_k$ is attained at an extreme point. 

Relabel the coordinates so that $x^\ast$ corresponds to index $k$. Then
\[
\overline{\Omega}_k
=
\left\{
\nu\in\Delta^{k-1}:
\nu_i\geq \frac1k \text{ for } i=1,\ldots,k-1,
\quad
0\leq \nu_k\leq \frac1k
\right\}.
\]

The extreme points of $\overline{\Omega}_k$ are determined by which inequalities are tight, so the active constraints are:
\[
\nu_i = 1/k \;\;(i = 1, \ldots, k-1), \qquad \nu_k = 0 \;\text{or}\; \nu_k = 1/k, \qquad \textstyle\sum_i \nu_i = 1.
\]
An extreme point of a polytope in $\mathbb{R}^k$ has at least $k$ active constraints. Among these, the simplex equality is always active. The remaining $k - 1$ active constraints fix all but one $\nu_i$ to either $1/k$ or $0$. We have two cases:
\begin{enumerate}
    \item[(a)] $\nu_k = 0$. Then $k - 2$ of the indices $i \in \{1, \ldots, k-1\}$ are fixed at $1/k$, and the remaining one absorbs the residual mass: $\nu_1 = 1 - (k-2)/k = 2/k$ (WLOG relabel this to index $1$).
    \item[(b)] $\nu_k = 1/k$. Then $k - 2$ of the indices $i \in \{1, \ldots, k-1\}$ are fixed at $1/k$, and the remaining one is WLOG $\nu_1 = 1 - (k-1)/k = 1/k$. This makes $\nu = \mathrm{Unif}(A)$.
\end{enumerate}
Case (b) attains 0, so the maximum is attained at case (a). 

At the case (a) extreme point $\nu^\ast = (2/k,\, 1/k,\, \ldots,\, 1/k,\, 0)$ with $k - 2$ copies of $1/k$,
\begin{align*}
H(\nu^\ast) &\;=\; -\frac{2}{k} \log \frac{2}{k} \;-\; (k-2) \cdot \frac{1}{k} \log \frac{1}{k} \;-\; 0 \cdot \log 0 \\
&\;=\; \log k \;-\; \frac{2}{k}\, \log 2,
\end{align*}
Therefore
\[
\log k - H(\nu^\ast) \;=\; \frac{2}{k}\, \log 2,
\]
and by convexity,
\[
\sup_{\nu \in \overline{\Omega}_k} \bigl(\log k - H(\nu)\bigr) \;=\; \frac{2}{k}\, \log 2.
\]
Multiplying by $s \leq 1$ gives $R_p(\mathcal{R}) \leq (2/k) \log 2$, establishing the first inequality of the proposition.

For sharpness at $k = 2$ we take $V = \{x_0, x^\ast\}$ with $p(x_0) = 1 - \epsilon$, $p(x^\ast) = \epsilon$ for any criterion whose rejection region contains $x^\ast$ when $p(x^\ast) = \epsilon$ (for example, additive relaxed with $t = 1 - 2\epsilon$, or strict greedy). The active set is $A = \{x_0, x^\ast\}$ and
\[
R_p(\mathcal{R}) \;=\; G\bigl(1 - \epsilon,\, \epsilon\bigr) \;=\; (1 - \epsilon) \log\bigl(2(1-\epsilon)\bigr) + \epsilon \log(2\epsilon) \;\xrightarrow{\epsilon \to 0}\; \log 2.
\]
Hence the upper bound is sharp in the limit.
\end{proof}

\subsection{Greedy Decoding}
\label{sec:greedy}

The first setting that we apply the single-token acceptance result is to the strict greedy token acceptance case. Under strict greedy decoding, the draft token is accepted if and only if it matches the target argmax $x_0$ which we assume is unique. Let
\[
x_1 \;:=\; \arg\max_{v \neq x_0} p(v), \qquad \gamma_p \;:=\; p(x_0) - p(x_1),
\]
so $\gamma_p$ is the target margin. The rejection region is
\[
\mathcal{R}_{\mathrm{greedy}} \;=\; V \setminus \{x_0\} \;=\; \{v \in V : p(v) \leq p(x_1)\}.
\]
Theorem~\ref{thm:single} therefore specializes immediately to the following exact certificate.

\begin{corollary}[Exact KL certificate for greedy agreement]
\label{cor:exact-greedy-certificate}
Let $p \in \Delta^{|V|-1}$ have a unique argmax $x_0$, and let $x_1 \in \arg\max_{v \neq x_0} p(v)$. Define
\[
M^\ast(p) \;:=\; p(x_0) \log\!\frac{2\,p(x_0)}{p(x_0) + p(x_1)} \;+\; p(x_1) \log\!\frac{2\,p(x_1)}{p(x_0) + p(x_1)}.
\]
Then
\[
\inf\bigl\{\,\mathrm{KL}(p \,\Vert\, q) \;:\; \arg\max_v q(v) \cap \mathcal{R}_{\mathrm{greedy}} \neq \emptyset \,\bigr\} \;=\; M^\ast(p),
\]
and the infimum is attained by $q^\ast$ with $q^\ast(x_0) = q^\ast(x_1) = \frac{1}{2}(p(x_0)+p(x_1))$ and $q^\ast(v) = p(v)$ for $v \notin \{x_0, x_1\}$. Consequently, if $\mathrm{KL}(p \,\Vert\, q) < M^\ast(p)$, the greedy speculative draft is accepted.
\end{corollary}
\begin{proof}
This is a direct application of Theorem~\ref{thm:single} with $\mathcal{R}_{\mathrm{greedy}}$. The most-probable element of $\mathcal{R}_{\mathrm{greedy}}$ is $x_1$. Since every $v \notin \{x_0, x_1\}$ has $p(v) \leq p(x_1) < \frac{1}{2}(p(x_0) + p(x_1)) = c(x_1)$, the active set in Lemma~\ref{lem:single-kkt} is exactly $A = \{x_0, x_1\}$. The resulting formula is $M(x_1) = G(p(x_0), p(x_1))$, which is precisely $M^\ast(p)$, and the stated $q^\ast$ is the corresponding minimizer from Lemma~\ref{lem:single-kkt}.
\end{proof}

Before we use $G$ to analyze the behavior of $M^\ast(p)$, we first give a simple sufficient condition for greedy acceptance that follows from Pinsker's inequality.

\begin{proposition}[Simple Margin Condition for Greedy Acceptance]
\label{prop:simple-margin}
Suppose that $\mathrm{KL}(p \,\Vert\, q) \leq \epsilon$. If $\gamma_p > \sqrt{2\epsilon}$, then the draft token is always accepted under greedy decoding.
\end{proposition}

\begin{proof}

By Pinsker's inequality, we know that:
\[
\mathrm{TV}(p, q) \;\leq\; \sqrt{\frac{1}{2}\, \mathrm{KL}(p \,\Vert\, q)}.
\]

In particular,
$ \forall v \in V: |p(v) - q(v)| \leq \sqrt{\epsilon/2}$

This implies that for any token $v$:
\[
q(v) \leq p(v) + \sqrt{\frac{\epsilon}{2}} \leq p(x_1) + \sqrt{\frac{\epsilon}{2}} < p(x_0) - \sqrt{\frac{\epsilon}{2}} \leq q(x_0)
\]

Thus, the (unique) mode of $q$ is the same as $p$, and the draft token is accepted under greedy decoding.

\end{proof}

We now show that the rate $\sqrt{2\epsilon}$ is in fact optimal up to an $O(\epsilon^{3/2})$ correction as $\epsilon \to 0$.

\begin{theorem}[Asymptotic tightness of $\sqrt{2\epsilon}$]
\label{thm:greedy-tightness}

For every $p$, $M^\ast(p) \geq g(\gamma_p)$, with equality when $p$ is supported on $\{x_0, x_1\}$. Moreover, Lemmas~\ref{lem:g-bounds} and~\ref{lem:g-inverse} imply that, as $\epsilon \to 0$,
\[
g^{-1}(\epsilon) \;=\; \sqrt{2\epsilon}\,\left(1 - \frac{\epsilon}{6} + O\left(\epsilon^2\right)\right).
\]
In particular, the following are true as guarantees for greedy acceptance:
\begin{enumerate}
    \item[(i)] (Sufficient.) If $\mathrm{KL}(p \,\Vert\, q) \leq \epsilon$ and $\gamma_p > g^{-1}(\epsilon)$, the greedy draft is always accepted.
    \item[(ii)] (Necessary.) For every $\epsilon > 0$ and every $d \leq \sqrt{2\epsilon/(1+2\epsilon)}$, there exist $p, q$ with $\gamma_p = d$ and $\mathrm{KL}(p \,\Vert\, q) \leq \epsilon$ such that $\arg\max q \neq \arg\max p$.
\end{enumerate}
\end{theorem}
\begin{proof}
 
By Corollary~\ref{cor:exact-greedy-certificate} we know that,
\[
M^\ast(p) = G(p(x_0), p(x_1)) \;=\; (p(x_0) + p(x_1))\, g\left(\frac{p(x_0) - p(x_1)}{p(x_0) + p(x_1)}\right) \;=\; s\, g\left(\frac{d}{s}\right),
\]
where $d = \gamma_p$ and $s = p(x_0) + p(x_1)$.

We now minimize over $s \in [d, 1]$ at fixed $d$. Differentiating,
\[
\frac{d}{ds}\bigl[\,s\, g\left(\frac{d}{s}\right)\,\bigr] \;=\; g\left(\frac{d}{s}\right) \;-\; \frac{d}{s}\, g'\left(\frac{d}{s}\right).
\]
A direct calculation gives $g'(x) = \frac{1}{2} \log\frac{1+x}{1-x}$, and substituting into $g(x) - x g'(x)$ yields
\[
g(x) - x g'(x) \;=\; \frac{1}{2} \log(1 - x^2).
\]
This is negative for all $x \in (0, 1)$, so $s\, g(d/s)$ is strictly decreasing in $s$. The minimum over $s \in [d, 1]$ is therefore attained at $s = 1$, giving
\[
M^\ast(p) \;\geq\; g(d) \;=\; g(\gamma_p),
\]
with equality when $s = 1$, i.e., when the support of $p$ is $\{x_0, x_1\}$.

The bounds for $g$ are proved in Lemma~\ref{lem:g-bounds}, and the expansion of $g^{-1}$ is proved in Lemma~\ref{lem:g-inverse}.

Now, for sufficiency, if $\mathrm{KL}(p \,\Vert\, q) \leq \epsilon$ and $\gamma_p > g^{-1}(\epsilon)$, then $g(\gamma_p) > \epsilon$, so by Corollary~\ref{cor:exact-greedy-certificate},
\[
M^\ast(p) \;\geq\; g(\gamma_p) \;>\; \epsilon \;\geq\; \mathrm{KL}(p \,\Vert\, q),
\]
The KL divergence between $p$ and $q$ is smaller than the minimum required for changing the mode, and hence the greedy draft is accepted.

For necessity, we provide a worst-case construction for $p$ and $q$. Fix $d \leq \sqrt{2\epsilon/(1+2\epsilon)}$. Construct $p, q$ on $\{x_0, x_1\}$ with
\[
p(x_0) \;=\; \frac{1+d}{2}, \qquad p(x_1) \;=\; \frac{1-d}{2}, \qquad q(x_0) \;=\; q(x_1) \;=\; \frac{1}{2}.
\]
Then $\gamma_p = d$. Computing gives $\mathrm{KL}(p \,\Vert\, q) = g(d)$, and the upper bound on $g$ yields $\mathrm{KL}(p \,\Vert\, q) \leq d^2/(2(1-d^2)) \leq \epsilon$ by the choice of $d$. Yet $\arg\max q = \{x_0, x_1\} \neq \{x_0\} = \arg\max p$, so the greedy draft is rejected.
\end{proof}

Corollary~\ref{cor:exact-greedy-certificate} tells us that the minimum KL divergence for which greedy acceptance can fail is $M^\ast(p)$, and it tells us that the worst case happens when $q$ distributes probability equally. In Theorem~\ref{thm:greedy-tightness}, the reduction from $M^\ast(p)$ to $g(\gamma_p)$ shows us that this cost depends to the leading order on the target's margin $\gamma_p$. 

This provides a useful insight into speculative decoding. Greedy speculative decoding depends extensively on the target distribution's margin. A confident target model that has a large margin will be an easier speculative target. 

On low margin target distributions, even a small draft-target divergence may be sufficient to cause the greedy criterion to reject. The KL loss during training or as an evaluation metric is not the same as promising greedy acceptance uniformly.

\subsection{Additive Relaxed Acceptance}
\label{sec:additive-relaxed-acceptance}

The first relaxation of greedy acceptance that we analyze is additive relaxation under a margin $t \in [0,1]$. Under this, the draft token $y = \arg\max_v q(v)$ is accepted whenever
\[
p(y) > p(x_0) - t.
\]
Equivalently, the rejection region is
\[
\mathcal{R}_t \;=\; \{v \in V : p(v) \leq p(x_0) - t\},
\]
so the general single-token theorem again applies directly to give us an exact KL threshold.

\begin{corollary}[Exact KL certificate for additive relaxed acceptance]
\label{cor:additive-exact}
Fix $t \in [0, 1]$ with $\mathcal{R}_t \neq \emptyset$, and let
\[
x^\ast_t \;:=\; \arg\max_{v \in \mathcal{R}_t} p(v).
\]
Let $c := c(x^\ast_t)$, $A := A(x^\ast_t)$, and $s := s(x^\ast_t)$ be as in Lemma~\ref{lem:single-kkt}. Then
\[
\inf\bigl\{ \mathrm{KL}(p \,\Vert\, q) : \arg\max_v q(v) \in \mathcal{R}_t \bigr\} \;=\; R_p(t) \;:=\; s \cdot \mathrm{KL}\!\left( \frac{p|_A}{s} \,\Bigm\Vert\, \mathrm{Unif}(A) \right).
\]
\end{corollary}

The next result gives the corresponding asymptotically sharp sufficient condition in terms of the relaxation margin $t$.

\begin{theorem}[Asymptotically optimal $t^2/2$ rate]
\label{thm:additive-rate}
For every $p$ and every $t \in [0, 1]$ with $\mathcal{R}_t \neq \emptyset$,
\[
R_p(t) \;\geq\; g(t) \;\geq\; \frac{t^2}{2} \;+\; \frac{t^4}{12},
\]
where $g(d) = \frac{1+d}{2}\log(1+d) + \frac{1-d}{2}\log(1-d)$ as in Theorem~\ref{thm:greedy-tightness}. In particular, if $\mathrm{KL}(p \,\Vert\, q) \leq \epsilon$ and
\[
t \;>\; \sqrt{2 \epsilon}\,\Bigl(1 - \frac{\epsilon}{6} + O(\epsilon^2)\Bigr),
\]
then the additive relaxed criterion accepts. The bound is asymptotically tight: there exist $p, q$ with $\mathrm{KL}(p \,\Vert\, q) = g(t) = t^2/2 + O(t^4)$ and $\arg\max_v q(v) \in \mathcal{R}_t$.
\end{theorem}
\begin{proof}

By Corollary~\ref{cor:additive-exact}, $R_p(t) \;=\; M(x^\ast_t)$. Writing
\[
d \;:=\; p(x_0) - p(x^\ast_t), \qquad s \;:=\; p(x_0) + p(x^\ast_t),
\]
we have $d \geq t$ and, by Lemma~\ref{lem:two-token-bound},
\[
R_p(t) \;\geq\; G(p(x_0), p(x^\ast_t)) \;=\; s\, g\!\left(\frac{d}{s}\right) \;\geq\; g(d) \;\geq\; g(t),
\]
where the penultimate inequality is the monotonicity statement proved in Theorem~\ref{thm:greedy-tightness}, and the final inequality uses $d \geq t$ together with $g'(x) = \frac{1}{2}\log\frac{1+x}{1-x} > 0$ on $(0,1)$. Lemma~\ref{lem:g-bounds} then gives
\[
g(t) \;\geq\; \frac{t^2}{2} + \frac{t^4}{12}.
\]
 
The condition $g(t) > \epsilon$ ensures acceptance via $R_p(t) \geq g(t) > \epsilon \geq \mathrm{KL}(p \,\Vert\, q)$. Lemma~\ref{lem:g-inverse} gives
\[
g^{-1}(\epsilon) \;\leq\; \sqrt{2\epsilon}\,\Bigl(1 - \frac{\epsilon}{6} + O(\epsilon^2)\Bigr).
\]
Hence $t > \sqrt{2\epsilon}(1 - \epsilon/6 + O(\epsilon^2))$ implies acceptance.
 
For tightness, use the same two-point construction as in Theorem~\ref{thm:greedy-tightness}, replacing $d$ by $t$. Then $\mathrm{KL}(p \,\Vert\, q) = g(t) = t^2/2 + O(t^4)$ and the draft is rejected under the additive criterion.
\end{proof}

Under strict greedy decoding, we showed in Section~\ref{sec:greedy} that a training KL bound of $\epsilon$ certifies acceptance only when the target margin satisfies $\gamma_p > \sqrt{2\epsilon}(1 - \epsilon/6 + O(\epsilon^2))$. By Theorem~\ref{thm:additive-rate}, we know that additive relaxation thus replaces the target model margin with a user-controllable parameter. I.e., setting $t = \sqrt{2\epsilon}(1 - \epsilon/6 + O(\epsilon^2))$ guarantees acceptance regardless of the actual target margin. Therefore, additive relaxed acceptance converts a property of the target model ($\gamma_p$) into a property of the algorithm ($t$).

This has several advantages. Firstly, it allows us to guarantee acceptance in a wider range of target distributions, especially under low margin. It also enables the intuitive possibility that if some set of tokens are close enough to the target argmax, then they could be accepted as well, for example, if the top two tokens are very close in probability, it may be acceptable to accept either of them. In several circumstances, the target model cannot be retrained. This therefore allows us to improve acceptance rates by setting an acceptance margin instead of relying on the target model's margin, which may be small in many cases.

However, the relaxation changes the model's output characteristics since it allows for different trajectories of accepted tokens. This may have implications for the quality of the generated text, and it may be desirable to keep $t$ as small as possible while still increasing acceptance length.

\subsection{Multiplicative Relaxed Acceptance}

Under multiplicative relaxation with factor $\alpha \in (0,1]$, the draft token is accepted whenever either it matches the target argmax or its target probability exceeds an $\alpha$-fraction of the target maximum. Equivalently, rejection occurs when the draft argmax lies in the lower level set
\[
\mathcal{R}^\times_\alpha \;:=\; \bigl\{ v \in V \setminus \{x_0\} : p(v) \leq \alpha\, p(x_0) \bigr\}.
\]
Assuming $\mathcal{R}^\times_\alpha \neq \emptyset$, this is again a lower level set of $p$, so the general single-token KL threshold specializes directly.

\begin{corollary}[Exact KL certificate for multiplicative relaxed acceptance]
\label{cor:mult-exact}
Let
\[
x^\ast_\alpha \;:=\; \arg\max_{v \in \mathcal{R}^\times_\alpha} p(v),
\]
and let $c := c(x^\ast_\alpha)$, $A := A(x^\ast_\alpha)$, and $s := s(x^\ast_\alpha)$ be as in Lemma~\ref{lem:single-kkt}. Then
\[
\inf\bigl\{ \mathrm{KL}(p \,\Vert\, q) : \arg\max_v q(v) \in \mathcal{R}^\times_\alpha \bigr\} \;=\; R^\times_p(\alpha) \;:=\; s \cdot \mathrm{KL}\!\left( \frac{p|_A}{s} \,\Bigm\Vert\, \mathrm{Unif}(A) \right).
\]
\end{corollary}
\begin{proof}
Apply Theorem~\ref{thm:single} with rejection region $\mathcal{R}^\times_\alpha$.
\end{proof}

While the exact certificate has the same structure as in the previous cases, the resulting general lower bound is different because the rejection threshold also depends on $p(x_0)$.

\begin{theorem}[Asymptotic rate for multiplicative relaxed acceptance]
\label{thm:mult-rate}
For every $p$ and every $\alpha \in (0, 1]$ with $\mathcal{R}^\times_\alpha \neq \emptyset$,
\[
R^\times_p(\alpha) \;\geq\; p(x_0) \left[ \log\frac{2}{1+\alpha} + \alpha \log\frac{2\alpha}{1+\alpha} \right] \;=\; p(x_0) \left[ \frac{(1-\alpha)^2}{4} + \frac{(1-\alpha)^3}{8} + O\bigl((1-\alpha)^4\bigr) \right]
\]
as $\alpha \to 1$. In particular, if $\mathrm{KL}(p \,\Vert\, q) \leq \epsilon$ and
\[
\epsilon \;<\; p(x_0) \left[ \log\frac{2}{1+\alpha} + \alpha \log\frac{2\alpha}{1+\alpha} \right],
\]
then the multiplicative criterion accepts. The bound is asymptotically tight: there exist $p, q$ with $\mathrm{KL}(p \,\Vert\, q) = p(x_0) \bigl[\log\frac{2}{1+\alpha} + \alpha \log\frac{2\alpha}{1+\alpha}\bigr]$ and $\arg\max_v q(v) \in \mathcal{R}^\times_\alpha$.
\end{theorem}
\begin{proof}
By Corollary~\ref{cor:mult-exact}, $R^\times_p(\alpha) = M(x^\ast_\alpha)$. By the same relaxation as in Theorem~\ref{thm:additive-rate}, we see that:
\[
R^\times_p(\alpha) \;\geq\; G\bigl(p(x_0), \alpha p(x_0)\bigr)
\]
by the threshold form of Lemma~\ref{lem:two-token-bound}, since $\mathcal{R}^\times_\alpha \subseteq \{v : p(v) \leq \alpha p(x_0)\}$. Writing
\[
g_\times(b) \;:=\; \log\frac{2}{1+b} + b \log\frac{2 b}{1 + b},
\]
this lower bound becomes
\[
R^\times_p(\alpha) \;\geq\; p(x_0)\, g_\times(\alpha).
\]
Taylor-expanding $g_\times$ at $\alpha = 1$,
\[
g_\times(\alpha) \;=\; \frac{(1-\alpha)^2}{4} + \frac{(1-\alpha)^3}{8} + O\bigl((1-\alpha)^4\bigr).
\]

For tightness, take $V = \{x_0, x_1\}$ and
\[
p \;=\; \Bigl(\frac{1}{1+\alpha}, \frac{\alpha}{1+\alpha}\Bigr), \qquad q \;=\; \Bigl(\frac{1}{2}, \frac{1}{2}\Bigr).
\]
Then $p(x_1) = \alpha\, p(x_0)$, so $x_1 \in \mathcal{R}^\times_\alpha$. Under the worst-case tie-breaking convention, $x_1 \in \arg\max_v q(v)$ triggers rejection. The KL is
\[
\mathrm{KL}(p \,\Vert\, q) \;=\; \frac{1}{1+\alpha}\log\frac{2}{1+\alpha} + \frac{\alpha}{1+\alpha}\log\frac{2\alpha}{1+\alpha} \;=\; p(x_0)\, g_\times(\alpha),
\]
matching the lower bound.

\end{proof}

The two relaxed criteria differ in how they convert a KL bound into an acceptance guarantee. Additive relaxed acceptance depends only on the margin parameter $t$, while the multiplicative relaxed acceptance depends on the top token probability $p(x_0)$ as well as the factor $\alpha$.

Multiplicative relaxation is therefore more conservative on less confident distributions but more relaxed on confident target distributions. On the other hand, the additive criterion is uniform. This means that the multiplicative criterion may be more effective for a user that wants to control the acceptance rate based on target model confidence. 

\subsection{Top-m Relaxed Acceptance}
\label{sec:topm}

A natural relaxed acceptance setting in practice, for example in TensorRT-LLM \citep{nvidia_tensorrt_llm_2026}, is to accept the draft token if it is among the top $m$ tokens of the target distribution and also satisfies a multiplicative relaxation threshold (we also consider additive thresholds). 

Let $x^{[1]}, x^{[2]}, \ldots$ enumerate $V$ in non-increasing order of $p$-value. The top-$m$ test is satisfied when the draft token $y$ satisfies $p(y) > p(x^{[m]})$. The additive relaxation test is satisfied when $p(y) > p(x_0) - t$. The multiplicative relaxation test is satisfied when $p(y) > \alpha\, p(x_0)$.

\begin{corollary}[Relaxed Top-m threshold acceptance]
\label{cor:relaxed-topm-threshold}
The rejection region of the top-$m$ relaxed criterion is the level set $\mathcal{R}_\theta = \{v \in V: p(v) \leq \theta\}$ with $\theta \;=\; \max\bigl( p(x_0) - t,\; p(x^{[m]}) \bigr)$ for the additive relaxation, and $\theta = \max\bigl(\alpha\, p(x_0),\, p(x^{[m]})\bigr)$ for the multiplicative relaxation.

Let $x^\ast_\theta := \arg\max_{v \in \mathcal{R}_\theta} p(v)$ and define the \emph{threshold margin} $\gamma_\theta := p(x_0) - \theta$. Then Theorem~\ref{thm:single} gives the exact certificate $R_p(\theta) := M(x^\ast_\theta)$, and by the threshold form of Lemma~\ref{lem:two-token-bound},
\[
R_p(\theta) \;\geq\; G\bigl(p(x_0), \theta\bigr) \;=\; \bigl(p(x_0) + \theta\bigr)\, g\!\left( \frac{\gamma_\theta}{p(x_0) + \theta} \right) \;\geq\; \frac{\gamma_\theta^2}{2\bigl(p(x_0) + \theta\bigr)} \;+\; \frac{\gamma_\theta^4}{12\bigl(p(x_0) + \theta\bigr)^3}.
\]
\end{corollary}

\begin{proof}
This is a direct application of Theorem~\ref{thm:single} with rejection region $\mathcal{R}_\theta$.
\end{proof}

This mixed criterion is governed by which threshold is more restrictive. Indeed, if $p(x^{[m]}) > p(x_0) - t$, then the top-$m$ condition dominates the additive relaxation, so the rejection region is $\{v : p(v) \leq p(x^{[m]})\}$ and the threshold margin is $p(x_0) - p(x^{[m]})$. An example of this is when $\max_v p(v) \leq t$. On the other hand, if we consider the tight construction for additive relaxed acceptance from Theorem~\ref{thm:additive-rate}, the rejection region is $\{v : p(v) \leq p(x_0) - t\}$ and the threshold margin is $t$. The same can be seen in the multiplicative case.

\subsection{Entropy Based Acceptance}
\label{sec:entropy}

Methods such as Medusa and Hydra accept draft tokens via an entropy-dependent threshold \citep{cai2024medusa,ankner2024hydra}: a token $y$ is accepted if\footnote{We state the criterion with strict inequality. With a non-strict acceptance condition, our tightness construction approaches the bound in a limit. In practice, exact equality is exceptionally unlikely, so the difference is negligible.}
\[
p(y) \;>\; \min\bigl( \epsilon_0,\; \delta_0\, e^{-H(p)} \bigr) \;=:\; \theta_H,
\]
where $H(p)$ is the Shannon entropy of the target distribution and $\epsilon_0, \delta_0 > 0$ are hyperparameters.

The key observation is that, for a fixed target $p$, the threshold $\theta_H$ is a constant. We assume that $\theta_H < p(x_0)$ for otherwise this acceptance criterion always rejects any draft token trivially. The criterion is therefore an absolute-threshold criterion as well and Theorem~\ref{thm:single} applies directly.

\begin{corollary}[Typical acceptance]
\label{cor:medusa}
The rejection region of the entropy-based criterion is $\mathcal{R}_{\theta_H}$ with $\theta_H = \min(\epsilon_0, \delta_0 e^{-H(p)}) < p(x_0)$, and Theorem~\ref{thm:single} gives the exact certificate $R_p(\theta_H)$ together with the bound
\[
R_p(\theta_H) \;\geq\; (p(x_0) + \theta_H) g\!\left( \frac{\gamma_H}{p(x_0) + \theta_H} \right) \;\geq\; \frac{\gamma_{H}^2}{2\bigl(p(x_0) + \theta_H\bigr)} \;+\; \frac{\gamma_H^4}{12\bigl(p(x_0) + \theta_H\bigr)^3}, \gamma_H := p(x_0) - \theta_H,
\]
\end{corollary}

\section{Tree Based Acceptance}
\label{sec:tree-based-acceptance}

\subsection{Setup and Notation}
\label{sec:tree-setup}

We restrict our exploration to the simplest tree-based decoding scheme where at each level, the drafter proposes the top $m$ tokens under $q$, and the target verifies all $m$ candidates in parallel. This section analyzes the local acceptance condition at a single tree level. The single token case of Section~\ref{sec:greedy} corresponds to $m = 1$.

Following Sections~\ref{sec:greedy} and~\ref{sec:additive-relaxed-acceptance}, we work at a fixed conditioning context, with $p$ and $q$ the target and draft distributions over the vocabulary $V$. Let $x_0 := \arg\max_v p(v)$ as before. Write the tokens of $V \setminus \{x_0\}$ in non-increasing order of $p$-value as $x_{(1)}, x_{(2)}, \ldots$, so that $p(x_{(1)}) \geq p(x_{(2)}) \geq \cdots$.

The drafter selects $\mathrm{top}_m(q) := \{v_1, \ldots, v_m\}$, the $m$ tokens with highest $q$-value. The target accepts the level if $x_0 \in \mathrm{top}_m(q)$; otherwise it rejects. Equivalently, rejection occurs when at least $m$ tokens have $q$-value at least $q(x_0)$. Under the worst-case tie-breaking convention of Section~\ref{sec:prelim-spec-dec}, the rejection region is
\[
\mathcal{R}^{\text{tree}}_m \;:=\; \bigl\{ q \in \Delta^{|V|-1} \;:\; \bigl|\{v \neq x_0 : q(v) \geq q(x_0)\}\bigr| \geq m \bigr\}.
\]

We define the \emph{$m$-th order margin} of $p$ as
\[
\gamma_m \;:=\; p(x_0) - p(x_{(m)}),
\]
i.e., the gap between the target's argmax and its $(m+1)$-th most probable token. For $m = 1$, $\gamma_1 = \gamma_p$ recovers the strict greedy margin.

\subsection{Exact Certificate for top-m acceptance.}
\label{sec:tree-acceptance}

We characterize the minimum KL divergence required to push $x_0$ out of the top $m$ tokens of $q$. The construction is similar to Corollary~\ref{cor:additive-exact}, but with the rejection region defined by an $m$-token constraint. This results in a similar expression for the KL threshold, but with a different re-arrangement of probability mass, namely starting from the least probable of the top $m+1$ tokens instead of the most.

\begin{theorem}[Exact KL certificate for tree-based greedy decoding]
\label{thm:tree-exact}
Let $S^\ast := \{x_{(1)}, \ldots, x_{(m)}\}$ be the $m$ most probable tokens of $V \setminus \{x_0\}$. Let $r$ be the unique solution to
\begin{equation}
p(x_0) - r \;=\; \sum_{i=1}^{m} \bigl(r - p(x_{(i)})\bigr)_+.
\label{eq:r-defining-eq}
\end{equation}
and define $A := \{x_0\} \cup \{v \in S^\ast : p(v) < r\}$, $s := \sum_{v \in A} p(v)$, $k := |A|$. Then
\[
\inf\bigl\{ \mathrm{KL}(p \,\Vert\, q) : q \in \mathcal{R}^{\text{tree}}_m \bigr\} \;=\; R^{\text{tree}}_p(m) \;:=\; s \cdot \mathrm{KL}\!\left( \frac{p|_A}{s} \,\Bigm\Vert\, \mathrm{Unif}(A) \right).
\]
\end{theorem}
\begin{proof}

For any set $S \subseteq V \setminus \{x_0\}$ with $|S| = m$, define
\[
M(S) \;:=\; \min_q \mathrm{KL}(p \,\Vert\, q) \quad \text{subject to} \quad q(v) \geq q(x_0) \;\;\text{for all } v \in S.
\]
A $q \in \mathcal{R}^{\text{tree}}_m$ has $m$ tokens with $q$-value $\geq q(x_0)$; taking $S$ to be any such set of $m$ tokens shows $q$ is feasible for $M(S)$. If $q$ is feasible for some $M(S),$ then it has at least $m$ tokens with $q$-value $\geq q(x_0)$ and thus $q \in \mathcal{R}^{\text{tree}}_m$. Hence
\[
\inf\bigl\{ \mathrm{KL}(p \,\Vert\, q) : q \in \mathcal{R}^{\text{tree}}_m \bigr\} \;=\; \min_{|S| = m,\, S \subseteq V \setminus \{x_0\}} M(S).
\]

Let us now fix $S$. We show that the minimizer of $M(S)$ is
\[
q^\ast(v) \;=\; \begin{cases} r_S & v \in A_S, \\ p(v) & v \notin A_S, \end{cases}
\]
where $A_S := \{x_0\} \cup \{v \in S : p(v) < r_S\}$ and $r_S$ is the unique solution of the equation \eqref{eq:r-defining-eq} with $S^*$ replaced by $S$ which exists by Lemma~\ref{lem:r-uniqueness}. 

We can check the KKT conditions as in Lemma~\ref{lem:single-kkt}. Use the constraints $q(x_0)-q(v) \leq 0, v \in S$
with multipliers $\lambda_v \geq 0$. The Lagrange multiplier for the simplex condition is $\mu=1$ once again. For each $v \in A_S \setminus \{x_0\}$, the constraint is active, since $q^\ast(v)=r_S=q^\ast(x_0)$, and we set $\lambda_v = 1-\frac{p(v)}{r_S} \geq 0.$

For each \(v \in S \setminus A_S\), we have \(p(v)\geq r_S\), hence
\[
q^\ast(v)=p(v)\geq r_S=q^\ast(x_0),
\]
and we set \(\lambda_v=0\). For \(v \notin S \cup \{x_0\}\), there is no inequality multiplier, and stationarity gives \(q^\ast(v)=p(v)\). Finally, the stationarity condition at \(x_0\) is
\[
-\frac{p(x_0)}{r_S}+1+\sum_{v \in A_S\setminus\{x_0\}}
\left(1-\frac{p(v)}{r_S}\right)=0,
\]
which is exactly equivalent to the defining equation for \(r_S\). Hence dual feasibility, complementary slackness, and stationarity are satisfied. The construction in Lemma~\ref{lem:r-uniqueness} preserves primal feasibility as well, so \(q^\ast\) is the minimizer of \(M(S)\).

Substituting $q^\ast$ into the KL objective, the terms outside $A_S$ vanish, and using $r_S = (\sum_{v \in A_S} p(v))/|A_S|$ gives
\[
M(S) \;=\; \Bigl(\sum_{v \in A_S} p(v)\Bigr) \cdot \mathrm{KL}\!\left( \frac{p|_{A_S}}{\sum_{v \in A_S} p(v)} \,\Bigm\Vert\, \mathrm{Unif}(A_S) \right).
\]

We show $S^\ast = \{x_{(1)}, \ldots, x_{(m)}\}$ minimizes $M(S)$. Suppose $v \in S$, $w \notin S \cup \{x_0\}$, with $p(w) \geq p(v)$. Let $S' = (S \setminus \{v\}) \cup \{w\}$. Let $q$ be the minimizer for $S$.

We have two cases. If $q(w) \geq q(x_0)$, then $q$ is already feasible for $M(S')$, and so $M(S') \leq M(S)$. 

Otherwise, $q(w) < q(x_0)$ and by feasibility of $q$, $q(v) \geq q(x_0)$. So if we define $q'$ by swapping $q(v)$ and $q(w)$, we have $q'(w) = q(v) \geq q(x_0)$, so $q'$ is feasible for $M(S')$. A direct computation gives
\[
\mathrm{KL}(p \,\Vert\, q) - \mathrm{KL}(p \,\Vert\, q') \;=\; \bigl(p(w) - p(v)\bigr) \log\frac{q(v)}{q(w)} \;\geq\; 0,
\]
since both factors are non-negative. Hence $M(S') \leq M(S)$. Iterating, the minimum is attained at $S^\ast$, the $m$ tokens of largest $p$-value outside $\{x_0\}$. The corresponding $r$, $A$, and threshold match the statement.
\end{proof}

\subsection{Lower bound}

We have now obtained an exact certificate $R^{\text{tree}}_p(m)$. To provide further interpretability, we now derive an explicit lower bound based only on the tree margin $\gamma_m$, which generalizes the strict greedy rate of Section~\ref{sec:greedy}.

For this proof, we use the Donsker-Varadhan variational representation of KL divergence, which gives us: for any function $h: V \to \mathbb{R}$,
\begin{equation}
\mathrm{KL}(p \,\Vert\, q) \;\geq\; \mathbb{E}_p[h] - \log \mathbb{E}_q[e^h].
\label{eq:dv}
\end{equation}
By choosing $h$ appropriately and optimizing over a free parameter, we obtain a clean closed-form lower bound.

\begin{theorem}[Tree-margin Lower Bound via Donsker--Varadhan]
\label{thm:tree-dv}
For every $p$ and every $m \geq 1$ with $\mathcal{R}^{\text{tree}}_m \neq \emptyset$,
\[
R^{\text{tree}}_p(m) \;\geq\; B_m(\gamma_m),
\]
where, for $\gamma \in [0, 1)$,
\[
B_m(\gamma) \;:=\; (1 + \gamma) \log \tau_0 \;-\; \log \frac{\tau_0^2 + (m-1) \tau_0 + 1}{m+1}
\]
and $\tau_0 > 1$ is the unique positive root of
\[
(1 - \gamma)\tau^2 - \gamma(m-1)\tau - (1 + \gamma) \;=\; 0.
\]
The bound is exactly attainable: for every $\gamma \in (0, 1)$, there exist $p, q$ with $\gamma_m = \gamma$ and $\mathrm{KL}(p \,\Vert\, q) = B_m(\gamma)$.

Asymptotically, as $\gamma \to 0$,
\[
B_m(\gamma) \;=\; \frac{m+1}{4}\gamma^2 \;+\; \frac{(5-m)(m+1)^2}{192}\gamma^4 \;+\; O(\gamma^6).
\]
In particular, if $\mathrm{KL}(p \,\Vert\, q) \leq \epsilon$ and $\gamma_m > \sqrt{\frac{4\epsilon}{m+1}}\,(1 + O(\epsilon))$, the draft is accepted.
\end{theorem}

\begin{proof}
The Donsker-Varadhan representation~\eqref{eq:dv} gives a lower bound on $\mathrm{KL}(p \,\Vert\, q)$ for any choice of $h: V \to \mathbb{R}$. For $\tau > 1$, we use the function
\[
h(v) \;=\; \begin{cases} 2 \log \tau & v = x_0, \\ 0 & v = x_{(m)}, \\ \log \tau & \text{otherwise.} \end{cases}
\]

We first compute $\mathbb{E}_p[h]$. The contribution from $x_0$ is $2 p(x_0) \log \tau$, from $x_{(m)}$ is $0$, and from all other tokens is $(1 - p(x_0) - p(x_{(m)})) \log \tau$. Summing gives
\[
\mathbb{E}_p[h] \;=\; \bigl(1 + p(x_0) - p(x_{(m)})\bigr) \log \tau \;=\; (1 + \gamma_m) \log \tau.
\]

Next, we bound $\mathbb{E}_q[e^h]$. By construction,
\[
\mathbb{E}_q[e^h] \;=\; \tau^2 q(x_0) + \tau\bigl(1 - q(x_0) - q(x_{(m)})\bigr) + q(x_{(m)}) \;=\; \tau + (\tau^2 - \tau) q(x_0) + (1 - \tau) q(x_{(m)}).
\]

By Theorem~\ref{thm:tree-exact}, we have
\[
R_p^{\mathrm{tree}}(m)=M(S^\ast),
\qquad
S^\ast=\{x_{(1)},\ldots,x_{(m)}\}.
\]
Therefore, it suffices to lower-bound \(\mathrm{KL}(p\Vert q)\) for arbitrary $q$ feasible for $M(S^\ast)$, i.e. for arbitrary $q$ satisfying
\[
q(x_{(i)})\geq q(x_0),\qquad i=1,\ldots,m.
\]
For such $q$, in particular, $q(x_{(m)})\geq q(x_0)$. With $\tau > 1$, the coefficient $(1 - \tau)$ is negative, so
\[
\mathbb{E}_q[e^h] \;\leq\; \tau + (\tau^2 - \tau) q(x_0) + (1 - \tau) q(x_0) \;=\; \tau + (\tau - 1)^2 q(x_0).
\]
Furthermore, the $m+1$ tokens $\{x_0, x_{(1)}, \ldots, x_{(m)}\}$ all have $q$-value $\geq q(x_0)$, so we must have $q(x_0) \leq 1/(m+1)$. Hence
\[
\mathbb{E}_q[e^h] \;\leq\; \tau + \frac{(\tau - 1)^2}{m+1} \;=\; \frac{\tau^2 + (m-1)\tau + 1}{m+1}.
\]

Substituting both bounds into~\eqref{eq:dv} gives for every $\tau > 1$:
\[
\mathrm{KL}(p \,\Vert\, q) \;\geq\; (1 + \gamma_m) \log \tau \;-\; \log \frac{\tau^2 + (m-1)\tau + 1}{m+1}.
\]

We now optimize for $\tau$. Differentiating with respect to $\tau$, setting the derivative to zero, and rearranging:
\[
(1 - \gamma_m)\tau^2 \;-\; \gamma_m(m-1)\tau \;-\; (1 + \gamma_m) \;=\; 0.
\]
The product of roots is $-(1+\gamma_m)/(1-\gamma_m) < 0$ for $\gamma_m \in (0, 1)$, so this quadratic has exactly one positive root, which we denote $\tau_0$. $\tau_0 > 1$ since plugging in $\tau = 1$ gives us $-m\gamma_m < 0$ and the quadratic is positive for $\tau \rightarrow \infty$. 
\[
\tau_0 \;=\; \frac{\gamma_m(m-1) + \sqrt{\gamma_m^2(m-1)^2 + 4(1 - \gamma_m^2)}}{2(1 - \gamma_m)}.
\]

We verify that $\tau_0$ is a maximizer (not a minimizer) of the right-hand side. Let
\[
F(\tau) \;:=\; (1 + \gamma_m) \log \tau - \log \frac{\tau^2 + (m-1)\tau + 1}{m+1}.
\]
At the boundary $\tau = 1$, $F(1) = -\log\bigl((m+1)/(m+1)\bigr) = 0$ and
\[
F'(1) \;=\; (1 + \gamma_m) - \frac{m + 1}{m + 1} \;=\; \gamma_m \;>\; 0,
\]
so $F$ is increasing at $\tau = 1$. As $\tau \to \infty$, $F(\tau) \sim (\gamma_m - 1) \log \tau \to -\infty$ since $\gamma_m < 1$. Hence $F$ attains its maximum on $(1, \infty)$ at an interior critical point. Since $\tau_0$ is the unique such critical point, it is the maximizer, and $B_m(\gamma_m) := F(\tau_0)$ is the optimal lower bound from this family.

The resulting bound is $B_m(\gamma_m)$ as stated. We finally derive the series $B_m(\gamma)$ near $\gamma = 0$. Substituting the Puiseux expansion of $\tau_0$ at $\gamma = 0$, we get
\[
B_m(\gamma) \;=\; \frac{m+1}{4}\gamma^2 \;+\; \frac{(5-m)(m+1)^2}{192}\gamma^4 \;+\; O(\gamma^6).
\]
The bound $B_m(\gamma)$ can be attained exactly. For every $\gamma \in (0, 1)$, there exist $p, q$ with $\gamma_m = \gamma$ and $\mathrm{KL}(p \,\Vert\, q) = B_m(\gamma)$. Let $\tau = \tau_0(\gamma)$ and $A = \tau^2 + (m-1)\tau + 1$. Take $V = \{x_0, x_1, \ldots, x_m\}$ and
\[
p(x_0) \;=\; \frac{\tau^2}{A}, \qquad p(x_i) \;=\; \frac{\tau}{A} \;\;\text{for } i = 1, \ldots, m-1, \qquad p(x_m) \;=\; \frac{1}{A},
\]
with $q$ uniform on $V$, i.e.\ $q(v) = 1/(m+1)$ for every $v$. Then $\sum_v p(v) = A/A = 1$, and
\[
p(x_0) - p(x_m) \;=\; \frac{\tau^2 - 1}{A} \;=\; \gamma,
\]
where the last equality follows from the constraint at the optimum (equivalently, $\tau$ satisfies the quadratic by construction). Under the worst-case tie-breaking convention, $q$ uniform places $x_0$ outside $\mathrm{top}_m(q)$, so $q \in \mathcal{R}^{\text{tree}}_m$. Computing the KL,
\begin{align*}
\mathrm{KL}(p \,\Vert\, q) \;&=\; \frac{\tau^2}{A} \log \frac{(m+1)\tau^2}{A} \;+\; \frac{(m-1)\tau}{A} \log \frac{(m+1)\tau}{A} \;+\; \frac{1}{A} \log \frac{m+1}{A} \\
&=\; \frac{2\tau^2 + (m-1)\tau}{A} \log \tau \;+\; \frac{\tau^2 + (m-1)\tau + 1}{A} \log \frac{m+1}{A} \\
&=\; (1 + \gamma) \log \tau \;-\; \log \frac{A}{m+1} \;=\; B_m(\gamma),
\end{align*}
where the second line groups the $\log \tau$ and $\log((m+1)/A)$ contributions, and the third uses the optimality identity $(2\tau^2 + (m-1)\tau)/A = 1 + \gamma$.

For the acceptance threshold, if $B_m(\gamma_m) > \epsilon \geq \mathrm{KL}(p \,\Vert\, q)$, the draft is accepted. Inverting the leading term gives the stated bound $\gamma_m > \sqrt{\frac{4\epsilon}{m+1}}(1 + O(\epsilon))$.
\end{proof}

Theorem~\ref{thm:tree-dv} shows that tree-branching at a single level reduces the required margin from $\sqrt{2\epsilon}$ to $\sqrt{4\epsilon/(m+1)}$ which is a $\sqrt{(m+1)/2}$ improvement for the same training KL bound. Equivalently, branching with $m$ candidates guarantees acceptance under a KL bound $(m+1)/2$ times larger than strict greedy decoding at the same margin.

\subsection{General upper bound}

Similar to the single-token case, we can also derive a general upper bound on the KL threshold for tree-based acceptance similar to Theorem~\ref{thm:log2-ceiling}. However, we observe that instead of $\log(2)$, the upper bound is $\log(m+1)$.

\begin{theorem}[Upper bound for tree-based certificates]
\label{thm:tree-upper}
Let $R^{\text{tree}}_p(m)$ be the exact tree certificate of Theorem~\ref{thm:tree-exact}. Then
\[
R^{\text{tree}}_p(m) \;\leq\; \log(m+1),
\]
and the bound is sharp in the limit $p(x_0) \to 1$ with $p(x_{(1)}), \ldots, p(x_{(m)}) \to 0$.
\end{theorem}

\begin{proof}
By Theorem~\ref{thm:tree-exact},
\[
R^{\text{tree}}_p(m) \;=\; s \cdot \mathrm{KL}\!\left( \frac{p|_A}{s} \,\Bigm\Vert\, \mathrm{Unif}(A) \right),
\]
where $A = \{x_0\} \cup \{v \in S^\ast : p(v) < r\}$ is the active set and $s = \sum_{v \in A} p(v)$. Writing $\nu := p|_A / s \in \Delta^{|A|-1}$,
\[
R^{\text{tree}}_p(m) \;=\; s\bigl(\log |A| - H(\nu)\bigr) \;\leq\; s \log |A| \;\leq\; \log |A|,
\]
using $H(\nu) \geq 0$ and $s \leq 1$. Since $A \subseteq \{x_0\} \cup S^\ast$ and $|S^\ast| = m$, we have $|A| \leq m + 1$, giving $R^{\text{tree}}_p(m) \leq \log(m+1)$.

For sharpness, take $V = \{x_0, x_1, \ldots, x_m\}$ with
\[
p(x_0) = 1 - m\epsilon, \qquad p(x_i) = \epsilon \;\;\text{for } i = 1, \ldots, m,
\]
for small $\epsilon > 0$. Here $S^\ast = \{x_1, \ldots, x_m\}$. Solving the defining equation~\eqref{eq:r-defining-eq} for $r$: since all $m$ tokens in $S^\ast$ have equal probability $\epsilon < r$, the active set is $A = \{x_0, x_1, \ldots, x_m\}$ with $|A| = m+1$, and $r = s/(m+1) = 1/(m+1)$. The certificate is
\[
R^{\text{tree}}_p(m) \;=\; (1 - m\epsilon) \log\bigl((m+1)(1 - m\epsilon)\bigr) \;+\; m\epsilon \log\bigl((m+1)\epsilon\bigr) \;\xrightarrow{\epsilon \to 0}\; \log(m+1),
\]
since the first term tends to $\log(m+1)$ and the second, $m\epsilon \log((m+1)\epsilon) \to 0$. Hence the bound is attained in the limit.
\end{proof}

\section{Empirical Study of Local Acceptance}

In many of our results, we have derived bounds in terms of various kinds of margins on the target model probability distribution. In this section, we provide empirical evidence that margins are typically large in practice, but also that small margins can occur in some cases.

\subsection{Setup}
We estimate the target model distributions of Qwen/Qwen3-1.7B and Qwen/Qwen3-4B models \citep{qwen3}. We use the UltraChat 200k dataset \citep{hfdataset} for evaluation. We choose 500 random prompts with at least 64 prompt tokens. Our generations last 128 steps. We finally store and analyze the target model distributions at each step of generation.

\subsection{Model Distributions}

We first analyze the distribution of the top token probability $p(x_0)$ across decoding steps in Table~\ref{tab:topk-cumprob} and Figure~\ref{fig:qwen-cumprob-top1}.

We first observe that target models are typically very confident. Indeed, the mean top token probability is $\approx 0.85$ for Qwen3-1.7B and $\approx 0.87$ for Qwen3-4B, and the median is even higher at $\approx 0.97$ and $\approx 0.99$ respectively. The gap between the mean and median indeed suggests a long tail of low-margin steps which is where the relaxed criteria and tree-based acceptance can provide the most benefit.

We also observe that the median cumulative probability of the top-3 tokens for both models rounds to 1, meaning that almost all of the probability mass concentrates on a few tokens on typical steps. Even on the hard tail of low-margin steps (p5), the top 5 tokens carry $\geq 0.79$ of the mass and the top 10 carry $\geq 0.90$.

These observations do not by themselves imply high acceptance for an arbitrary drafter. Instead, they show that many target model distributions have local acceptance certificates. Consequently, a drafter that achieves small KL on such prefixes would be certified for acceptance by our bounds. The low-probability tail is where relaxed or tree-based criteria can improve the KL thresholds and allow higher acceptance rates.

\begin{table}[t]
\centering
\caption{Distribution of top-$k$ cumulative target probability $\sum_{i=1}^{k} p_{(i)}$ over decoding steps, where $p_{(i)}$ is the $i$-th largest target probability.}
\label{tab:topk-cumprob}
\small
\begin{tabular}{l rrrr c rrrr}
\toprule
& \multicolumn{4}{c}{Qwen3-1.7B} & & \multicolumn{4}{c}{Qwen3-4B} \\
\cmidrule{2-5} \cmidrule{7-10}
$k$ & Mean & Med & p5 & p25 & & Mean & Med & p5 & p25 \\
\midrule
1   & 0.845 & 0.969 & 0.343 & 0.846 & & 0.871 & 0.988 & 0.391 & 0.886 \\
3   & 0.963 & 1.000 & 0.715 & 0.984 & & 0.974 & 1.000 & 0.787 & 0.992 \\
5   & 0.980 & 1.000 & 0.792 & 0.995 & & 0.988 & 1.000 & 0.852 & 0.998 \\
10  & 0.991 & 1.000 & 0.901 & 0.999 & & 0.995 & 1.000 & 0.927 & 1.000 \\
25  & 0.997 & 1.000 & 0.964 & 1.000 & & 0.998 & 1.000 & 0.976 & 1.000 \\
\bottomrule
\end{tabular}
\end{table}

\begin{figure*}[t]
\centering
\begin{subfigure}[t]{0.49\textwidth}
\centering
\includegraphics[width=\linewidth]{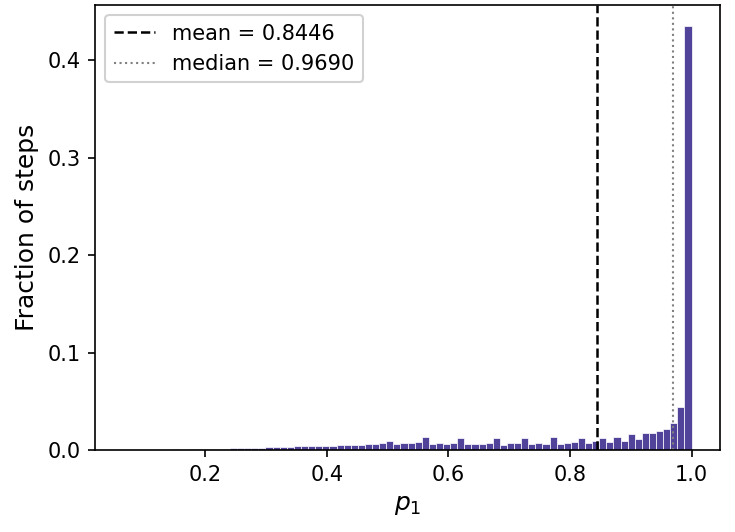}
\caption{Qwen3-1.7B}
\end{subfigure}\hfill
\begin{subfigure}[t]{0.49\textwidth}
\centering
\includegraphics[width=\linewidth]{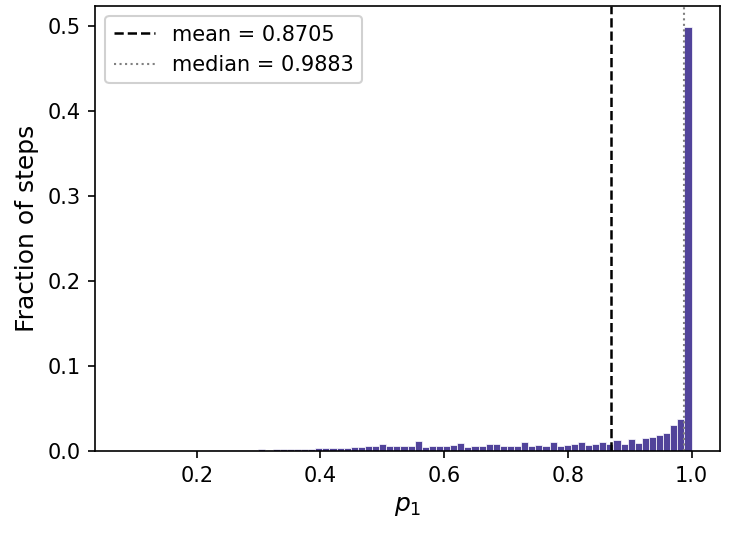}
\caption{Qwen3-4B}
\end{subfigure}
\caption{Distribution of the probability of the argmax token of both models. More than 40\% of steps have a top-1 probability close to 1, however there is a long tail of low-probability steps.}
\label{fig:qwen-cumprob-top1}
\end{figure*}

\subsection{KL Certificates}

We also analyze the exact certificates for each acceptance criterion. For each criterion, we report the Mean, Median, p5 and p25 quantiles of the KL threshold distribution across decoding steps in Table~\ref{tab:main-certificates}. We focus on the lower quantiles because the upper quantiles are typically very large and less informative about the differences between criteria. Figure~\ref{fig:qwen-frac-greedy} shows the fraction of steps with certificates larger than $\epsilon$.

Higher certificates mean that the criterion is more tolerant of training error. We can see that additive relaxed acceptances are both more tolerant than strict greedy acceptance, although median for $t = 0.1$ is very close to the greedy case. The same can be seen for multiplicative relaxed acceptance with $\alpha = 0.5$. These relaxed acceptances mainly increase the certificate size on the tail of low margin steps as shown by increases in p5 quantiles.

On the other hand, Medusa's entropy based acceptance ($\epsilon_0=0.1,\delta_0=0.09$) substantially lifts the tail, with p5 $\approx 0.375$ for both models. This is because the entropy-based threshold is typically much smaller than the top-1 probability, so the margin $\gamma_H$ is typically large even on low-margin steps.

Finally, tree branching at $m = 2$ doubles the median certificate compared to strict greedy acceptance, and substantially lifts the tail as well. Increasing $m$ further continues to increase the certificates, although with diminishing returns.

Note the steep drop at $\log(2)$ for the greedy case which empirically validates the certificate upper bound in Theorem~\ref{thm:log2-ceiling} and corresponds to the large number of confident steps with $p(x_0) \approx 1$ which matches our limiting construction in Theorem~\ref{thm:log2-ceiling}. 

\begin{table}[t]
\centering
\caption{Empirical distribution of exact KL certificates over decoding steps. Higher certificates mean more tolerant of training error. We report the Mean, Median, 5th Percentile and 25th Percentile certificate values.}
\label{tab:main-certificates}
\small
\begin{tabular}{l rrrr c rrrr}
\toprule
& \multicolumn{4}{c}{Qwen3-1.7B} & & \multicolumn{4}{c}{Qwen3-4B} \\
\cmidrule{2-5} \cmidrule{7-10}
Criterion & Mean & Med & p5 & p25 & & Mean & Med & p5 & p25 \\
\midrule
Strict greedy             & 0.438 & 0.559 & 0.006 & 0.160 & & 0.468 & 0.619 & 0.007 & 0.217 \\
Additive ($t = 0.1$)      & 0.445 & 0.559 & 0.023 & 0.178 & & 0.475 & 0.619 & 0.025 & 0.236 \\
Additive ($t = 0.3$)      & 0.482 & 0.573 & 0.102 & 0.272 & & 0.505 & 0.625 & 0.107 & 0.315 \\
Mult. ($\alpha = 0.5$)    & 0.464 & 0.560 & 0.061 & 0.235 & & 0.492 & 0.619 & 0.066 & 0.288 \\
Mult. ($\alpha = 0.1$)    & 0.555 & 0.610 & 0.279 & 0.435 & & 0.573 & 0.643 & 0.296 & 0.465 \\
Medusa ($\epsilon_0=0.1,\delta_0=0.09$) & 0.598 & 0.646 & 0.375 & 0.522 & & 0.606 & 0.664 & 0.378 & 0.536 \\
Tree ($m = 2$)            & 0.772 & 0.962 & 0.071 & 0.448 & & 0.820 & 1.036 & 0.099 & 0.538 \\
Tree ($m = 4$)            & 1.213 & 1.456 & 0.249 & 0.862 & & 1.276 & 1.541 & 0.327 & 0.995 \\
Tree ($m = 8$)            & 1.736 & 2.038 & 0.531 & 1.355 & & 1.814 & 2.126 & 0.663 & 1.520 \\
\bottomrule
\end{tabular}
\end{table}

\begin{figure*}[t]
\centering
\begin{subfigure}[t]{0.49\textwidth}
\centering
\includegraphics[width=\linewidth]{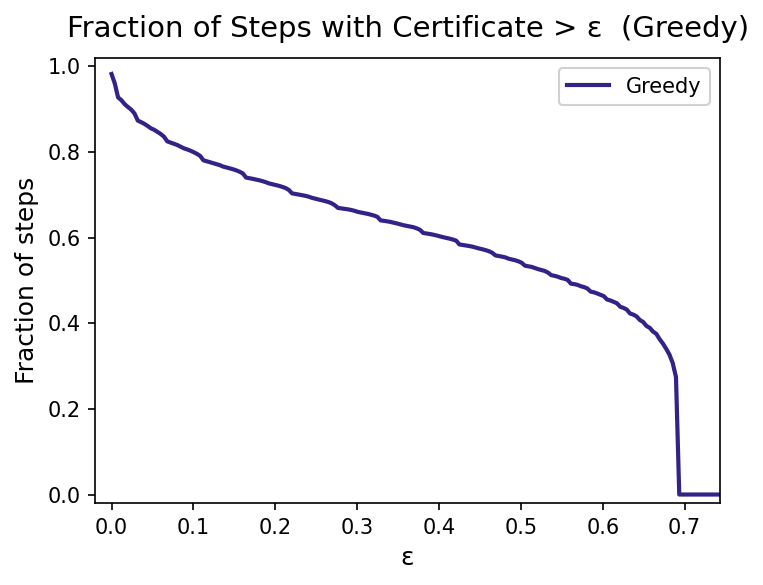}
\caption{Qwen3-1.7B}
\end{subfigure}\hfill
\begin{subfigure}[t]{0.49\textwidth}
\centering
\includegraphics[width=\linewidth]{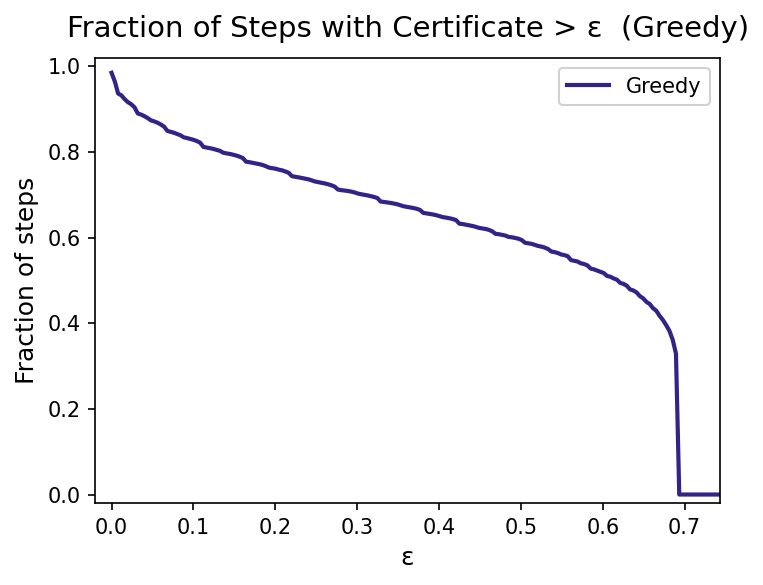}
\caption{Qwen3-4B}
\end{subfigure}
\caption{Both models have similar output distributions and hence have similar greedy certificate distributions. We note that there is a sharp drop at $\log(2)$ because the certificate is bounded above by $\log(2)$ which corresponds to the large number of very certain generation steps. We maintain the same x-limits across all graphs for easier comparison.}
\label{fig:qwen-frac-greedy}
\end{figure*}

In Table~\ref{tab:ablation-relaxed}, we analyze the sensitivity of the additive and multiplicative certificates to the threshold parameters $t$ and $\alpha$. We can see that small values of $t$ and large values of $\alpha$ are close to the greedy case, while large $t$ and small $\alpha$ substantially increase the certificates. We note that the values of $t$ smaller than 0.1 and $\alpha$ larger than 0.5 have negligible effect on the certificates.

This phenomenon is also illustrated in Figure~\ref{fig:qwen-frac-additive} and Figure~\ref{fig:qwen-frac-multiplicative}, where we plot the fraction of steps with certificates larger than $\epsilon$ for different values of $t$ and $\alpha$. We can see that the curves for $t = 0.01$ and $\alpha = 0.9$ are almost identical to the greedy case, while larger $t$ and smaller $\alpha$ substantially increase the fraction of steps with large certificates. We also note that the same drop at $\log(2)$ occurs for all curves, which corresponds to the large number of very certain steps with $p(x_0) \approx 1$ consistent with Theorem~\ref{thm:log2-ceiling}.

\begin{figure*}[t]
\centering
\begin{subfigure}[t]{0.49\textwidth}
\centering
\includegraphics[width=\linewidth]{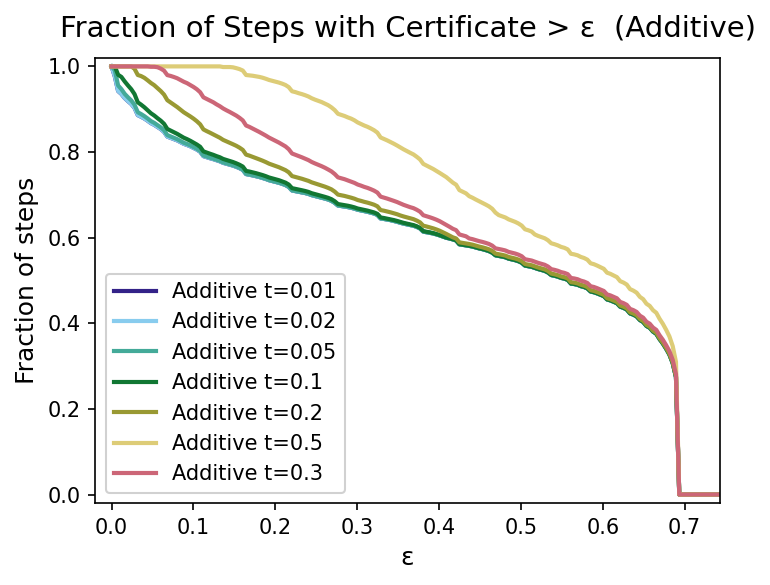}
\caption{Qwen3-1.7B}
\end{subfigure}\hfill
\begin{subfigure}[t]{0.49\textwidth}
\centering
\includegraphics[width=\linewidth]{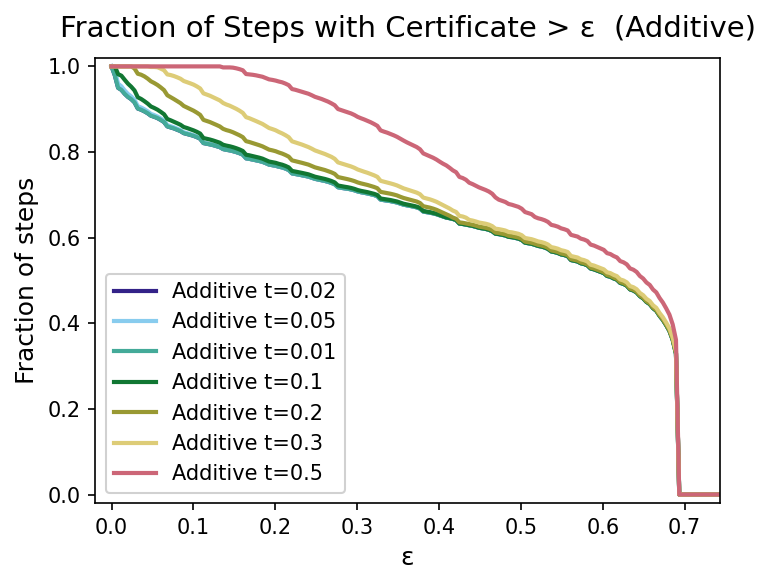}
\caption{Qwen3-4B}
\end{subfigure}
\caption{Additive relaxed certificate comparison across different values of $t$. We see the same drop at $\log(2)$ as in the greedy case corresponding to the limiting case in Theorem~\ref{thm:log2-ceiling}.}
\label{fig:qwen-frac-additive}
\end{figure*}

\begin{figure*}[t]
\centering
\begin{subfigure}[t]{0.49\textwidth}
\centering
\includegraphics[width=\linewidth]{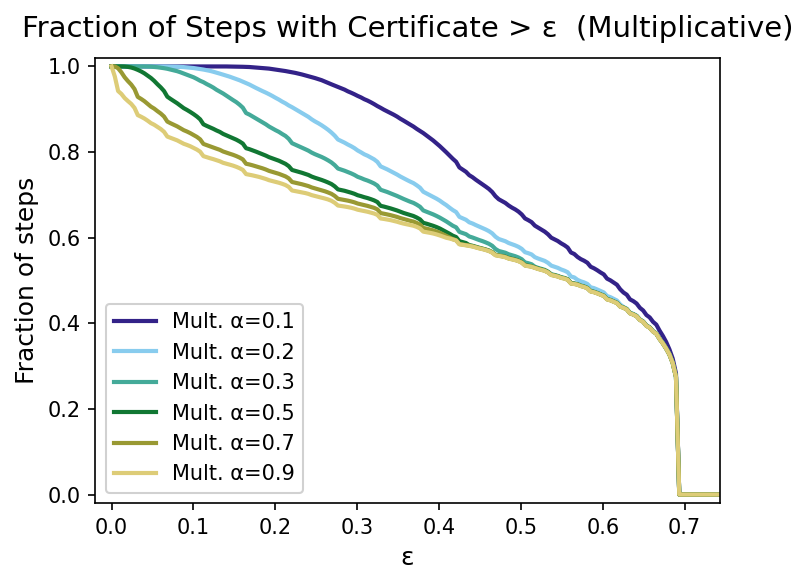}
\caption{Qwen3-1.7B}
\end{subfigure}\hfill
\begin{subfigure}[t]{0.49\textwidth}
\centering
\includegraphics[width=\linewidth]{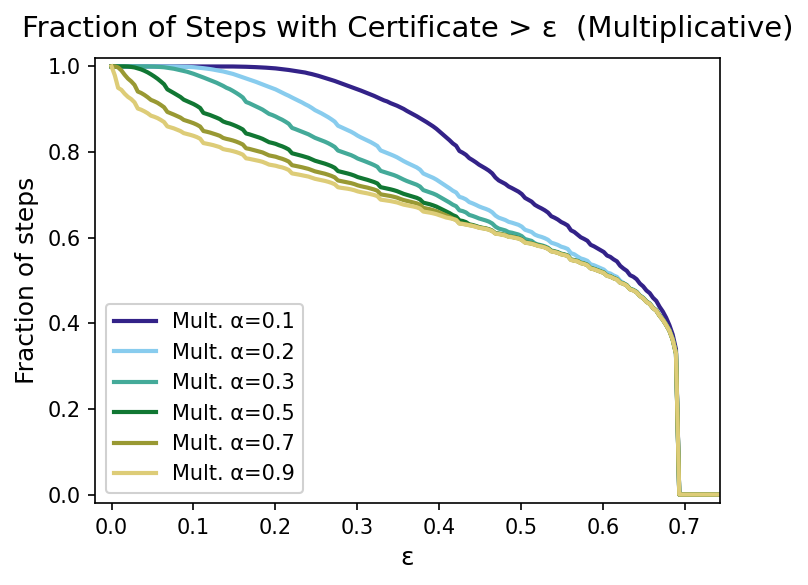}
\caption{Qwen3-4B}
\end{subfigure}
\caption{Multiplicative relaxed certificate comparison across different values of $\alpha$. We see the same drop at $\log(2)$ as in the greedy case corresponding to the limiting case in Theorem~\ref{thm:log2-ceiling}.}
\label{fig:qwen-frac-multiplicative}
\end{figure*}

\begin{table}[t]
\centering
\caption{Sensitivity of additive and multiplicative certificates to the threshold parameters.}
\label{tab:ablation-relaxed}
\small
\begin{minipage}[t]{0.47\textwidth}
\centering
\textbf{Additive certificate}\par\vspace{0.25em}
\begin{tabular}{l rr}
\toprule
$t$ & p5 (1.7B) & p5 (4B) \\
\midrule
0.01 & 0.008 & 0.008 \\
0.05 & 0.011 & 0.014 \\
0.10 & 0.023 & 0.025 \\
0.20 & 0.056 & 0.061 \\
0.30 & 0.102 & 0.107 \\
0.50 & 0.217 & 0.218 \\
\bottomrule
\end{tabular}
\end{minipage}
\begin{minipage}[t]{0.47\textwidth}
\centering
\textbf{Multiplicative certificate}\par\vspace{0.25em}
\begin{tabular}{l rr}
\toprule
$\alpha$ & p5 (1.7B) & p5 (4B) \\
\midrule
0.9 & 0.008 & 0.008 \\
0.7 & 0.027 & 0.030 \\
0.5 & 0.061 & 0.066 \\
0.3 & 0.125 & 0.141 \\
0.2 & 0.176 & 0.196 \\
0.1 & 0.279 & 0.296 \\
\bottomrule
\end{tabular}
\end{minipage}
\end{table}

For the mixed certificate of top-m and additive relaxed acceptance case, we notice that for small $t$ as we expected, the certificate is mostly determined by the additive relaxation threshold, while for smaller $m$ and larger $t$ the top-m gate has a greater effect. Namely, we find that the certificates for $t=0.1,m=2$ and $t=0.2,m=2$ are mostly constrained by $m=2$ while for larger $m$ the additive branch is more constraining. In addition, for $t=0.2$ the effect of the top-m condition can be seen to be more pronounced as $m$ increases as there is a small distinction between $m=4$ and $m=6$ in both models.

\begin{table}[t]
\centering
\caption{Top-$m$ additive relaxed acceptance 5th percentile of the certificate.}
\label{tab:ablation-topm}
\small
\begin{tabular}{l rrrr c rrrr}
\toprule
& \multicolumn{4}{c}{p5 (1.7B)} & & \multicolumn{4}{c}{p5 (4B)} \\
\cmidrule{2-5} \cmidrule{7-10}
$t$ & $m=2$ & $m=4$ & $m=6$ & $m=8$ & & $m=2$ & $m=4$ & $m=6$ & $m=8$ \\
\midrule
0.01 & 0.008 & 0.008 & 0.008 & 0.008 & & 0.011 & 0.011 & 0.011 & 0.011 \\
0.05 & 0.011 & 0.012 & 0.012 & 0.012 & & 0.014 & 0.015 & 0.015 & 0.015 \\
0.10 & 0.018 & 0.023 & 0.024 & 0.024 & & 0.023 & 0.027 & 0.027 & 0.027 \\
0.20 & 0.037 & 0.054 & 0.056 & 0.057 & & 0.045 & 0.060 & 0.063 & 0.063 \\
\bottomrule
\end{tabular}
\end{table}

Finally, we also examine the Tree-based acceptance certificates. We also observe the $\log(m+1)$ upper bound from Theorem~\ref{thm:tree-upper} in practice for $m$-branching which is indeed higher than $\log(2)$ for single-token criteria.

\begin{figure*}[t]
\centering
\begin{subfigure}[t]{0.49\textwidth}
\centering
\includegraphics[width=\linewidth]{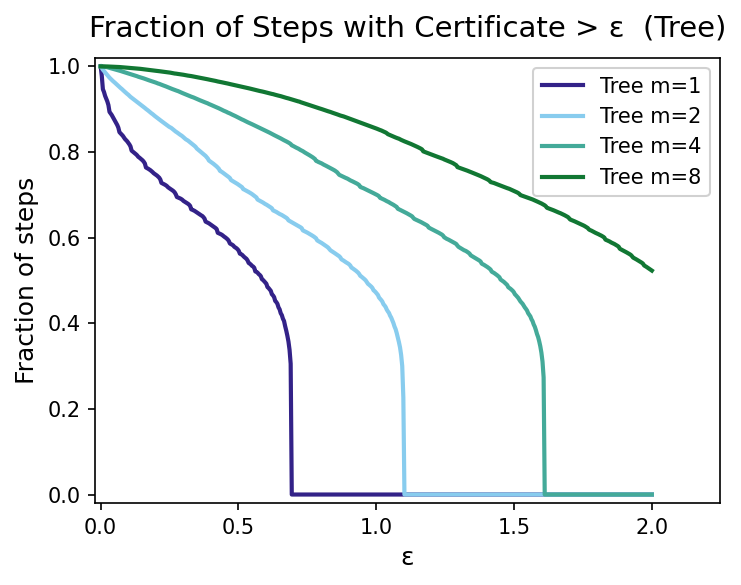}
\caption{Qwen3-1.7B}
\end{subfigure}\hfill
\begin{subfigure}[t]{0.49\textwidth}
\centering
\includegraphics[width=\linewidth]{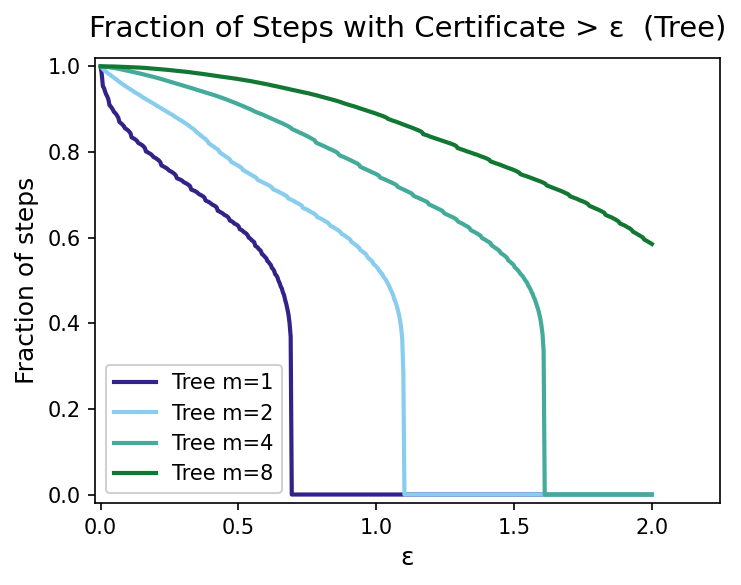}
\caption{Qwen3-4B}
\end{subfigure}
\caption{Tree certificate comparison.}
\label{fig:qwen-frac-tree}
\end{figure*}

\subsection{Guaranteed Acceptance Lengths}
\label{sec:acceptance-length}

In practice, one of the most common metrics for evaluating the quality of a draft model is the expected acceptance length, i.e., the expected number of tokens accepted before the first rejection. We can use our certificates to derive lower bounds on this quantity assuming that each draft can have unlimited length similar to \citet{yin2024a}. 

For a fixed $\epsilon$, call a position $\epsilon$-certified for a criterion if its exact certificate exceeds $\epsilon$. We compute the average length of consecutive $\epsilon$-certified positions along the generated trajectories. If a drafter satisfies $\mathrm{KL}(p_h\Vert q_h) < \epsilon$ at every position in such a run, then the corresponding tokens are guaranteed to be accepted.

In Tables~\ref{tab:mal-1p7b} and~\ref{tab:mal-4b}, we report these certifiable acceptance lengths assuming that at each step we are able to guarantee $\mathrm{KL}(p, q) < \epsilon$. We vary $\epsilon$ from 0.01 to 0.69 (which is approximately $\log(2)$, the universal ceiling for single-token criteria) to see how the expected certifiable acceptance length changes as we relax the training KL bound. A token is accepted if its exact certificate exceeds $\epsilon$. The number of steps we sampled exceeds 63000.

\begin{table}[t]
\centering
\caption{Mean certifiable acceptance length for Qwen3-1.7B as a function of the training KL bound $\epsilon$. Each entry is the empirical mean accepted run starting from a random decoding position, where a position is accepted if its exact certificate exceeds $\epsilon$. We capped the run lengths at 100. The ones exceeding 100 are marked with $>$100.}
\label{tab:mal-1p7b}
\small
\setlength{\tabcolsep}{4pt}
\begin{tabular}{l rrrrrrrr}
\toprule
& \multicolumn{8}{c}{$\epsilon$} \\
\cmidrule(lr){2-9}
Criterion & 0.01 & 0.05 & 0.10 & 0.20 & 0.30 & 0.40 & 0.50 & 0.69 \\
\midrule
Strict greedy             & 12.06 &  6.51 &  4.83 &  3.53 &  2.90 &  2.49 &  2.16 & 1.36 \\
Additive ($t = 0.1$)      & 34.57 &  8.37 &  5.41 &  3.71 &  2.97 &  2.51 &  2.16 & 1.37 \\
Additive ($t = 0.3$)      &$>$100 &$>$100 & 17.79 &  5.53 &  3.55 &  2.73 &  2.23 & 1.38 \\
Mult.\ ($\alpha = 0.5$)   &$>$100 & 26.16 &  8.42 &  4.46 &  3.27 &  2.61 &  2.17 & 1.37 \\
Mult.\ ($\alpha = 0.1$)   &$>$100 &$>$100 &$>$100 & 63.18 & 13.07 &  5.21 &  2.86 & 1.38 \\
Medusa ($\epsilon_0=0.1,\delta_0 = 0.09$) &$>$100 &$>$100 &$>$100 &$>$100 & 45.38 & 12.57 &  4.61 & 1.43 \\
\midrule
Tree ($m = 2$)            & 57.15 & 22.25 & 13.57 &  8.02 &  5.81 &  4.36 &  3.55 & 2.72 \\
Tree ($m = 4$)            &$>$100 & 68.99 & 41.51 & 21.95 & 14.14 & 10.31 &  7.88 & 5.31 \\
Tree ($m = 8$)            &$>$100 &$>$100 & 90.03 & 54.67 & 36.46 & 25.27 & 18.78 &11.93 \\
\bottomrule
\end{tabular}
\end{table}

\begin{table}[t]
\centering
\caption{Mean certifiable acceptance length for Qwen3-4B as a function of the training KL bound $\epsilon$. Each entry is the empirical mean accepted run starting from a random decoding position, where a position is accepted if its exact certificate exceeds $\epsilon$. We capped the run lengths at 100. The ones exceeding 100 are marked with $>$100.}
\label{tab:mal-4b}
\small
\setlength{\tabcolsep}{4pt}
\begin{tabular}{l rrrrrrrr}
\toprule
& \multicolumn{8}{c}{$\epsilon$} \\
\cmidrule(lr){2-9}
Criterion & 0.01 & 0.05 & 0.10 & 0.20 & 0.30 & 0.40 & 0.50 & 0.69 \\
\midrule
Strict greedy             & 13.65 &  7.42 &  5.60 &  4.07 &  3.31 &  2.81 &  2.44 & 1.47 \\
Additive ($t = 0.1$)      & 36.13 &  9.69 &  6.38 &  4.31 &  3.40 &  2.84 &  2.44 & 1.47 \\
Additive ($t = 0.3$)      &$>$100 & 98.09 & 19.67 &  6.39 &  4.05 &  3.07 &  2.51 & 1.48 \\
Mult.\ ($\alpha = 0.5$)   &$>$100 & 34.21 & 10.44 &  5.32 &  3.78 &  2.97 &  2.45 & 1.47 \\
Mult.\ ($\alpha = 0.1$)   &$>$100 &$>$100 &$>$100 & 74.15 & 16.24 &  6.30 &  3.30 & 1.49 \\
Medusa ($\epsilon_0=0.1,\delta_0=0.09$) &$>$100 &$>$100 &$>$100 & 99.57 & 45.57 & 13.33 &  5.06 & 1.55 \\
\midrule
Tree ($m = 2$)            & 64.57 & 28.39 & 17.25 & 10.24 &  7.30 &  5.27 &  4.19 & 3.16 \\
Tree ($m = 4$)            &$>$100 & 84.46 & 54.05 & 29.54 & 19.03 & 13.82 & 10.44 & 6.65 \\
Tree ($m = 8$)            &$>$100 &$>$100 &$>$100 & 73.77 & 50.31 & 35.21 & 26.67 &16.14 \\
\bottomrule
\end{tabular}
\end{table}

We observe that for very accurate draft models, usually unattainable in practice, the relaxed criteria can cause very long certifiable acceptance runs. However, in the more realistic regimes where $\epsilon$ is larger, we see certifiable acceptance length of around $2$ to $5$ for single-token criteria which are quite consistent with many speculative decoding papers \citep{li2026eagle,chen_dflash_2026}.

Some relaxed acceptance schemes can cause almost every draft to be accepted. This is not necessarily beneficial, since it may cause quality degradation in extreme cases, for instance when the KL bound $\epsilon$ is as large as $0.3$ under our Medusa settings.

Comparing the tree based acceptance to the single-token criteria, we see that it can substantially increase the certifiable acceptance lengths with also no output quality degradation as our tree-based analysis performs strict greedy acceptance. However, this also increases the computational requirements significantly. 

\FloatBarrier
\section{Conclusion}

In this paper, we connected speculative decoding to local stability. While much existing theory focuses on exact stochastic sampling and distribution preservation, many practical decoding variants are governed by deterministic local events such as whether the draft mode matches the target mode, whether a drafted token lies above a relaxed target threshold, or whether a draft tree covers the target greedy token. We showed that these events admit exact KL certificates. For several cases when a single token is drafted, we identified rejection regions that were lower-level sets of the target model distribution. This framework recovers strict greedy decoding, additive and multiplicative relaxation, top-(m) relaxed acceptance, and entropy-based acceptance as special cases. 

Our formulation also enables us to derive lower bounds on these KL thresholds for each criterion, and additionally construct examples that are asymptotically tight and in a few cases exact. We also derived universal upper bounds, showing that no single-token criterion of this form can have a KL threshold larger than $\log(2)$, and that for tree-based decoding with branching factor $m$ this ceiling extends to $\log(m+1)$. Both bounds are tight in the limit.

These results highlight the role of target margins. Greedy decoding is often difficult on target distributions with small margins. Additive relaxed acceptance replaces this uncontrolled margin by a user-chosen tolerance, while multiplicative relaxation provides an acceptance method based on model confidence. Drafting multiple tokens such as branching at each step reduces the margin requirement from $\sqrt{2\epsilon}$ to $\sqrt{4\epsilon/(m+1)}$. Empirical certificate diagnostics on Qwen3 models show that these effects are substantial in practice, particularly on the steps with low margin where greedy acceptance is least stable.

We would like to conclude this paper by posing some additional questions. The guarantees of this paper are local, in that they certify acceptance at a fixed prefix under a draft-target KL bound. A study on how successive drafts change these assumptions is an interesting direction and further bridges theory to practice. Furthermore, many tree-based decoding methods are dynamic, in that the tree structure can change at each draft step. Another question would be to develop a theory for these dynamic methods. Another question is whether other assumptions on the target distribution could lead to improved guarantees.

\bibliography{tmlr}

@misc{chen_dflash_2026,
	title = {{DFlash}: {Block} {Diffusion} for {Flash} {Speculative} {Decoding}},
	copyright = {Creative Commons Attribution 4.0 International},
	shorttitle = {{DFlash}},
	url = {https://arxiv.org/abs/2602.06036},
	doi = {10.48550/ARXIV.2602.06036},
	urldate = {2026-06-04},
	publisher = {arXiv},
	author = {Chen, Jian and Liang, Yesheng and Liu, Zhijian},
	year = {2026},
	note = {Version Number: 2},
}

@misc{an_pard-2_2026,
	title = {{PARD}-2: {Target}-{Aligned} {Parallel} {Draft} {Model} for {Dual}-{Mode} {Speculative} {Decoding}},
	copyright = {arXiv.org perpetual, non-exclusive license},
	shorttitle = {{PARD}-2},
	url = {https://arxiv.org/abs/2605.08632},
	doi = {10.48550/ARXIV.2605.08632},
	urldate = {2026-06-04},
	publisher = {arXiv},
	author = {An, Zihao and Liu, Taichi and Liu, Ziqiong and Li, Dong and Liu, Ruofeng and Barsoum, Emad},
	year = {2026},
	note = {Version Number: 1},
}

@misc{chen2023accelerating,
	title = {Accelerating {Large} {Language} {Model} {Decoding} with {Speculative} {Sampling}},
	copyright = {Creative Commons Attribution 4.0 International},
	url = {https://arxiv.org/abs/2302.01318},
	doi = {10.48550/ARXIV.2302.01318},
	urldate = {2026-06-04},
	publisher = {arXiv},
	author = {Chen, Charlie and Borgeaud, Sebastian and Irving, Geoffrey and Lespiau, Jean-Baptiste and Sifre, Laurent and Jumper, John},
	year = {2023},
	note = {Version Number: 1},
}

@inproceedings{miao_specinfer_2024,
	address = {La Jolla CA USA},
	title = {{SpecInfer}: {Accelerating} {Large} {Language} {Model} {Serving} with {Tree}-based {Speculative} {Inference} and {Verification}},
	isbn = {979-8-4007-0386-7},
	shorttitle = {{SpecInfer}},
	url = {https://dl.acm.org/doi/10.1145/3620666.3651335},
	doi = {10.1145/3620666.3651335},
	language = {en},
	urldate = {2026-06-05},
	booktitle = {Proceedings of the 29th {ACM} {International} {Conference} on {Architectural} {Support} for {Programming} {Languages} and {Operating} {Systems}, {Volume} 3},
	publisher = {ACM},
	author = {Miao, Xupeng and Oliaro, Gabriele and Zhang, Zhihao and Cheng, Xinhao and Wang, Zeyu and Zhang, Zhengxin and Wong, Rae Ying Yee and Zhu, Alan and Yang, Lijie and Shi, Xiaoxiang and Shi, Chunan and Chen, Zhuoming and Arfeen, Daiyaan and Abhyankar, Reyna and Jia, Zhihao},
	month = apr,
	year = {2024},
	pages = {932--949},
}

@inproceedings{li_eagle-2_2024,
	address = {Miami, Florida, USA},
	title = {{EAGLE}-2: {Faster} {Inference} of {Language} {Models} with {Dynamic} {Draft} {Trees}},
	shorttitle = {{EAGLE}-2},
	url = {https://aclanthology.org/2024.emnlp-main.422},
	doi = {10.18653/v1/2024.emnlp-main.422},
	language = {en},
	urldate = {2026-06-05},
	booktitle = {Proceedings of the 2024 {Conference} on {Empirical} {Methods} in {Natural} {Language} {Processing}},
	publisher = {Association for Computational Linguistics},
	author = {Li, Yuhui and Wei, Fangyun and Zhang, Chao and Zhang, Hongyang},
	year = {2024},
	pages = {7421--7432},
}

@misc{wertheimer2024combined,
	title = {Accelerating {Production} {LLMs} with {Combined} {Token}/{Embedding} {Speculators}},
	copyright = {Creative Commons Attribution 4.0 International},
	url = {https://arxiv.org/abs/2404.19124},
	doi = {10.48550/ARXIV.2404.19124},
	urldate = {2026-06-05},
	publisher = {arXiv},
	author = {Wertheimer, Davis and Rosenkranz, Joshua and Parnell, Thomas and Suneja, Sahil and Ranganathan, Pavithra and Ganti, Raghu and Srivatsa, Mudhakar},
	year = {2024},
	note = {Version Number: 2},
}

@inproceedings{leviathan2023fast, author = {Leviathan, Yaniv and Kalman, Matan and Matias, Yossi}, title = {Fast inference from transformers via speculative decoding}, year = {2023}, publisher = {JMLR.org}, booktitle = {Proceedings of the 40th International Conference on Machine Learning}, articleno = {795}, numpages = {13}, location = {Honolulu, Hawaii, USA}, series = {ICML'23} }

@inproceedings{cai2024medusa, author = {Cai, Tianle and Li, Yuhong and Geng, Zhengyang and Peng, Hongwu and Lee, Jason D. and Chen, Deming and Dao, Tri}, title = {MEDUSA: Simple LLM inference acceleration framework with multiple decoding heads}, year = {2024}, publisher = {JMLR.org}, booktitle = {Proceedings of the 41st International Conference on Machine Learning}, articleno = {203}, numpages = {27}, location = {Vienna, Austria}, series = {ICML'24} }

@inproceedings{
zhang2025learning,
title={Learning Harmonized Representations for Speculative Sampling},
author={Lefan Zhang and Xiaodan Wang and Yanhua Huang and Ruiwen Xu},
booktitle={The Thirteenth International Conference on Learning Representations},
year={2025},
url={https://openreview.net/forum?id=T9u56s7mbk}
}

@inproceedings{zhang2024draftverify,
    title = "Draft {\&} Verify: Lossless Large Language Model Acceleration via Self-Speculative Decoding",
    author = "Zhang, Jun  and
      Wang, Jue  and
      Li, Huan  and
      Shou, Lidan  and
      Chen, Ke  and
      Chen, Gang  and
      Mehrotra, Sharad",
    editor = "Ku, Lun-Wei  and
      Martins, Andre  and
      Srikumar, Vivek",
    booktitle = "Proceedings of the 62nd Annual Meeting of the Association for Computational Linguistics (Volume 1: Long Papers)",
    month = aug,
    year = "2024",
    address = "Bangkok, Thailand",
    publisher = "Association for Computational Linguistics",
    url = "https://aclanthology.org/2024.acl-long.607/",
    doi = "10.18653/v1/2024.acl-long.607",
    pages = "11263--11282"
}

@inproceedings{
liu2024kangaroo,
title={Kangaroo: Lossless Self-Speculative Decoding for Accelerating {LLM}s via Double Early Exiting},
author={Fangcheng Liu and Yehui Tang and Zhenhua Liu and Yunsheng Ni and Duyu Tang and Kai Han and Yunhe Wang},
booktitle={The Thirty-eighth Annual Conference on Neural Information Processing Systems},
year={2024},
url={https://openreview.net/forum?id=lT3oc04mDp}
}

@inproceedings{
li2026eagle,
title={{EAGLE}-3: Scaling up Inference Acceleration of Large Language Models via Training-Time Test},
author={Yuhui Li and Fangyun Wei and Chao Zhang and Hongyang Zhang},
booktitle={The Thirty-ninth Annual Conference on Neural Information Processing Systems},
year={2026},
url={https://openreview.net/forum?id=4exx1hUffq}
}

@inproceedings{
chen2024sequoia,
title={Sequoia: Scalable and Robust Speculative Decoding},
author={Zhuoming Chen and Avner May and Ruslan Svirschevski and Yu-Hsun Huang and Max Ryabinin and Zhihao Jia and Beidi Chen},
booktitle={The Thirty-eighth Annual Conference on Neural Information Processing Systems},
year={2024},
url={https://openreview.net/forum?id=rk2L9YGDi2}
}

@inproceedings{li2024eagle, author = {Li, Yuhui and Wei, Fangyun and Zhang, Chao and Zhang, Hongyang}, title = {EAGLE: speculative sampling requires rethinking feature uncertainty}, year = {2024}, publisher = {JMLR.org}, booktitle = {Proceedings of the 41st International Conference on Machine Learning}, articleno = {1162}, numpages = {14}, location = {Vienna, Austria}, series = {ICML'24} }

@inproceedings{
ankner2024hydra,
title={Hydra: Sequentially-Dependent Draft Heads for Medusa Decoding},
author={Zachary Ankner and Rishab Parthasarathy and Aniruddha Nrusimha and Christopher Rinard and Jonathan Ragan-Kelley and William Brandon},
booktitle={First Conference on Language Modeling},
year={2024},
url={https://openreview.net/forum?id=FbhjirzvJG}
}

@inproceedings{
sun2023spectr,
title={SpecTr: Fast Speculative Decoding via Optimal Transport},
author={Ziteng Sun and Ananda Theertha Suresh and Jae Hun Ro and Ahmad Beirami and Himanshu Jain and Felix Yu},
booktitle={Thirty-seventh Conference on Neural Information Processing Systems},
year={2023},
url={https://openreview.net/forum?id=SdYHLTCC5J}
}

@inproceedings{
ahn2023spectr,
title={SpecTr++: Improved transport plans for speculative decoding of large language models},
author={Kwangjun Ahn and Ahmad Beirami and Ziteng Sun and Ananda Theertha Suresh},
booktitle={NeurIPS 2023 Workshop Optimal Transport and Machine Learning},
year={2023},
url={https://openreview.net/forum?id=o8ZGNn9LpN}
}

@inproceedings{
yin2024a,
title={A Theoretical Perspective for Speculative Decoding Algorithm},
author={Ming Yin and Minshuo Chen and Kaixuan Huang and Mengdi Wang},
booktitle={The Thirty-eighth Annual Conference on Neural Information Processing Systems},
year={2024},
url={https://openreview.net/forum?id=wSqpNeMVLU}
}

@inproceedings{stern2018blockwise, author = {Stern, Mitchell and Shazeer, Noam and Uszkoreit, Jakob}, title = {Blockwise parallel decoding for deep autoregressive models}, year = {2018}, publisher = {Curran Associates Inc.}, address = {Red Hook, NY, USA}, booktitle = {Proceedings of the 32nd International Conference on Neural Information Processing Systems}, pages = {10107–10116}, numpages = {10}, location = {Montr\'{e}al, Canada}, series = {NIPS'18} }

@inproceedings{
an2026pard,
title={{PARD}: Accelerating {LLM} Inference with Low\nobreakdash-Cost {PAR}allel Draft Model Adaptation},
author={Zihao An and Huajun Bai and Ziqiong Liu and Dong Li and Emad Barsoum},
booktitle={The Fourteenth International Conference on Learning Representations},
year={2026},
url={https://openreview.net/forum?id=XbOyv7iVGL}
}

@inproceedings{zheng2024sglang,
title={{SGL}ang: Efficient Execution of Structured Language Model Programs},
author={Lianmin Zheng and Liangsheng Yin and Zhiqiang Xie and Chuyue Sun and Jeff Huang and Cody Hao Yu and Shiyi Cao and Christos Kozyrakis and Ion Stoica and Joseph E. Gonzalez and Clark Barrett and Ying Sheng},
booktitle={The Thirty-eighth Annual Conference on Neural Information Processing Systems},
year={2024},
url={https://openreview.net/forum?id=VqkAKQibpq}
}

@software{sglang_2026,
  author  = {{SGLang Team}},
  title   = {{SGLang: A high-performance serving framework for large language models and multimodal models}},
  year    = {2026},
  version = {0.5.12},
  url     = {https://github.com/sgl-project/sglang},
  note    = {Accessed: 2026-06-05}
}

@software{nvidia_tensorrt_llm_2026,
  author  = {{NVIDIA Corporation}},
  title   = {{TensorRT-LLM: A TensorRT toolbox for optimized large language model inference}},
  year    = {2026},
  version = {1.3.0rc16},
  url     = {https://github.com/NVIDIA/TensorRT-LLM},
  note    = {Accessed: 2026-06-05}
}

@inproceedings{holsman2025fuzzy,
    title = "Fuzzy Speculative Decoding for a Tunable Accuracy-Runtime Tradeoff",
    author = "Holsman, Maximilian  and
      Huang, Yukun  and
      Dhingra, Bhuwan",
    editor = "Che, Wanxiang  and
      Nabende, Joyce  and
      Shutova, Ekaterina  and
      Pilehvar, Mohammad Taher",
    booktitle = "Findings of the Association for Computational Linguistics: ACL 2025",
    month = jul,
    year = "2025",
    address = "Vienna, Austria",
    publisher = "Association for Computational Linguistics",
    url = "https://aclanthology.org/2025.findings-acl.1346/",
    doi = "10.18653/v1/2025.findings-acl.1346",
    pages = "26257--26273",
    ISBN = "979-8-89176-256-5"
}

@misc{zhong2025approxverify,
      title={Speeding up Speculative Decoding via Sequential Approximate Verification}, 
      author={Meiyu Zhong and Noel Teku and Ravi Tandon},
      year={2025},
      eprint={2502.04557},
      archivePrefix={arXiv},
      primaryClass={cs.LG},
      url={https://arxiv.org/abs/2502.04557}, 
}

@misc{song2026mars,
      title={MARS: Unleashing the Power of Speculative Decoding via Margin-Aware Verification}, 
      author={Jingwei Song and Xinyu Wang and Hanbin Wang and Xiaoxuan Lei and Bill Shi and Shixin Han and Eric Yang and Xiao-Wen Chang and Lynn Ai},
      year={2026},
      eprint={2601.15498},
      archivePrefix={arXiv},
      primaryClass={cs.LG},
      url={https://arxiv.org/abs/2601.15498}, 
}

@misc{qwen3,
      title={Qwen3 Technical Report}, 
      author={An Yang and Anfeng Li and Baosong Yang and Beichen Zhang and Binyuan Hui and Bo Zheng and Bowen Yu and Chang Gao and Chengen Huang and Chenxu Lv and Chujie Zheng and Dayiheng Liu and Fan Zhou and Fei Huang and Feng Hu and Hao Ge and Haoran Wei and Huan Lin and Jialong Tang and Jian Yang and Jianhong Tu and Jianwei Zhang and Jianxin Yang and Jiaxi Yang and Jing Zhou and Jingren Zhou and Junyang Lin and Kai Dang and Keqin Bao and Kexin Yang and Le Yu and Lianghao Deng and Mei Li and Mingfeng Xue and Mingze Li and Pei Zhang and Peng Wang and Qin Zhu and Rui Men and Ruize Gao and Shixuan Liu and Shuang Luo and Tianhao Li and Tianyi Tang and Wenbiao Yin and Xingzhang Ren and Xinyu Wang and Xinyu Zhang and Xuancheng Ren and Yang Fan and Yang Su and Yichang Zhang and Yinger Zhang and Yu Wan and Yuqiong Liu and Zekun Wang and Zeyu Cui and Zhenru Zhang and Zhipeng Zhou and Zihan Qiu},
      year={2025},
      eprint={2505.09388},
      archivePrefix={arXiv},
      primaryClass={cs.CL},
      url={https://arxiv.org/abs/2505.09388}, 
}

@misc{hfdataset,
      title={Enhancing Chat Language Models by Scaling High-quality Instructional Conversations}, 
      author={Ning Ding and Yulin Chen and Bokai Xu and Yujia Qin and Zhi Zheng and Shengding Hu and Zhiyuan Liu and Maosong Sun and Bowen Zhou},
      year={2023},
      eprint={2305.14233},
      archivePrefix={arXiv},
      primaryClass={cs.CL}
}
\bibliographystyle{tmlr}

\appendix
\section{Appendix}
 
In the appendix, we provide additional proofs for the results stated in the main text.

\subsection{Bounds and inverse expansion for $g$}
\label{app:g-bounds}

\begin{lemma}[Bounds for $g$]
\label{lem:g-bounds}
For every $d \in [0,1)$,
\[
\frac{d^2}{2} + \frac{d^4}{12} \;\leq\; g(d) \;\leq\; \frac{d^2}{2(1-d^2)}.
\]
\end{lemma}
\begin{proof}
The Taylor expansion of $g$ at the origin is
\[
g(d) \;=\; \sum_{i = 1}^{\infty} \frac{d^{2i}}{(2i-1)(2i)}.
\]
All coefficients are non-negative, so truncating after the $d^4$ term gives
\[
g(d) \;\geq\; \frac{d^2}{2} + \frac{d^4}{12}.
\]

For the upper bound, note that $g(0) = g'(0) = 0$ and
\[
g''(x) \;=\; \frac{1}{2(1+x)} + \frac{1}{2(1-x)} \;=\; \frac{1}{1-x^2}.
\]
Hence $g''(x) \leq 1/(1-d^2)$ for every $x \in [0,d]$. Integrating both sides twice from $0$ to $d$ gives:
\[
g(d) \;\leq\; \frac{d^2}{2(1-d^2)}.
\]
\end{proof}

\begin{lemma}[Expansion of $g^{-1}$] 
\label{lem:g-inverse} 
The function $g$ is strictly increasing on $[0,1)$, so $g^{-1}$ is well defined on $[0,\log 2)$. As $\eta \to 0$, 
\[ 
g^{-1}(\eta) = \sqrt{2\eta}\,\Bigl(1-\frac{\eta}{6}+O(\eta^2)\Bigr). 
\] 
\end{lemma} 
\begin{proof} 
Since 
\[ 
g'(d)=\frac12\log\frac{1+d}{1-d}, 
\] 
we have $g'(d)>0$ for every $d\in(0,1)$, and hence $g$ is strictly increasing. The Taylor expansion at the origin is 
\[ 
g(d)=\frac{d^2}{2}+\frac{d^4}{12}+O(d^6). 
\] 
Consider
\[ 
g(\sqrt{t}) \;=\; \sum_{i = 1}^{\infty} \frac{t^{i}}{(2i-1)(2i)}.
\] 
The power series inverse of $g(\sqrt{t})$ at $t = 0$ is:
\[ 
t \;=\; 2\eta-\frac{2}{3}\eta^2+O\left(\eta^3\right)
\] 
Taking square roots gives us
\[ 
d \;=\; \sqrt{2\eta}\, \left(1-\frac{\eta}{6}+O(\eta^2)\right), 
\] 
as claimed. 
\end{proof}

\subsection{Lemmas for Construction}

\begin{lemma}
\label{lem:existence-of-c}
Let $V$ be finite, fix $x_1 \in V$, and let $p:V \to \mathbb{R}$. Then there exists a unique real number $c=c(x_1)$ satisfying 
\[ 
c \;=\; \frac{p(x_1) + \sum_{v \neq x_1\,:\, p(v) > c} p(v)} {1 + \bigl|\{v \neq x_1 : p(v) > c\}\bigr|} \in [p(x_1), 1].
\]
\end{lemma}
\begin{proof}

Let us define
\[
F(t) := p(x_1)-t+\sum_{v\neq x_1} \max\{p(v)-t,0\}.
\]
Since $V$ is finite, $F$ is a continuous function of $t$. Moreover, if $s<t$, then
\[
F(s)-F(t)
= (t-s)+\sum_{v\neq x_1}
\Bigl(\max\{p(v)-s,0\}-\max\{p(v)-t,0\}\Bigr).
\]
Since each term in the sum is non-negative, $F(s) > F(t)$, so $F$ is strictly decreasing.

Also,
\[
\lim_{t\to -\infty} F(t)=+\infty,
\qquad
\lim_{t\to +\infty} F(t)=-\infty.
\]

Thus, there exists a $c$ such that $F(c)=0$ by the intermediate value theorem. Rewriting the equation $F(c)=0$ gives the desired fixed-point equation for $c$.

Since $F$ is strictly decreasing, the solution is unique.

Finally, since each $p(v) \in [0,1]$, the right-hand side of the fixed-point equation is a convex combination of numbers in $[p(x_1),1]$, so $c$ must also lie in $[p(x_1),1]$.
\end{proof}

\begin{lemma}
\label{lem:r-uniqueness}

For any $p$ and $m$ with $x_0,x_1,...,x_m$ satisfying $p(x_0) > p(x_1) \geq \cdots \geq p(x_m)$, there exists a unique solution $r$ to the equation

\[
p(x_0) - r \;=\; \sum_{i=1}^{m} \bigl(r - p(x_{i})\bigr)_+.
\]

\end{lemma}
\begin{proof}

The function 
\[
F(r) := p(x_0) - r - \sum_{i=1}^{m} \bigl(r - p(x_{i})\bigr)_+
\]
is continuous. It is differentiable on all but finitely many points (those being the $x_1, \ldots, x_m$). Where it is differentiable, the derivative is:
\[
F'(r) = -1 - \sum_{i=1}^{m} \mathbf{1}\{r > p(x_{i})\} \leq -1
\]
It is hence strictly decreasing. At the endpoints, we have $F(0) = p(x_0) > 0$ and $F(p(x_0)) = -\sum_{i=1}^{m} \bigl(p(x_0) - p(x_{i})\bigr)_+ < 0$. By the intermediate value theorem, there exists a solution to $F(r) = 0$ with $r \in (0, p(x_0))$, and by strict monotonicity, the solution is unique.
\end{proof}

\end{document}